\documentclass{article}

\newif\ifinappendix
\inappendixfalse

\usepackage{amsfonts}
\usepackage{amsmath}
\usepackage{amssymb}
\usepackage{amsthm}
\usepackage{appendix}
\usepackage{color}
\usepackage{float}
\usepackage{graphicx}
\usepackage{hyperref}
\usepackage{ltxtable}
\usepackage{mathtools}
\usepackage{mdwlist}
\usepackage{natbib,natbibspacing}
\usepackage{nicefrac}
\usepackage{subfigure}
\usepackage{url}
\usepackage{verbatim}
\usepackage{xcolor}
\usepackage{algorithm}    
\usepackage{algorithmic}

\makeatletter
\newtheorem*{rep@theorem}{\rep@title}
\newcommand{\newreptheorem}[2]{\newenvironment{rep#1}[1]{\def\rep@title{#2 \ref{##1}}\begin{rep@theorem}}{\end{rep@theorem}}}
\makeatother

\newtheorem{theorem}{Theorem}
\newreptheorem{theorem}{Theorem}

\newreptheorem{lemma}{Lemma}
\newtheorem{corollary}{Corollary}
\newreptheorem{corollary}{Corollary}

\theoremstyle{plain}
\newtheorem{thm}{Theorem}[section]
\newtheorem{lem}[thm]{Lemma}

\theoremstyle{definition}

\theoremstyle{remark}

\theoremstyle{plain}
\newreptheorem{thm}{Theorem}
\newreptheorem{lem}{Lemma}
\newreptheorem{cla}{Claim}
\newreptheorem{prop}{Proposition}
\theoremstyle{definition}
\newreptheorem{defn}{Definition}
\newreptheorem{conj}{Conjecture}
\newreptheorem{exmp}{Example}
\theoremstyle{remark}
\newreptheorem{case}{Case}

\numberwithin{equation}{section}

\newcolumntype{L}{>{\centering\arraybackslash} m{0.04\textwidth}}
\newcolumntype{S}{>{\centering\arraybackslash} m{0.32\textwidth}}

\newcommand{\todo}[1]{ \textcolor{red}{TODO: #1} }

\newcommand{\probability}[2][]{\mathrm{Pr}_{#1}\left\{#2\right\}}
\newcommand{\expectation}[2][]{\mathbb{E}_{#1}\left[#2\right]}

\newcommand{\argmin}[1]{\underset{#1}{\mathrm{argmin}} \:}
\newcommand{\argmax}[1]{\underset{#1}{\mathrm{argmax}} \:}
\newcommand{\inner}[2]{\left\langle {#1}, {#2} \right\rangle}
\newcommand{\norm}[1]{\left\lVert {#1} \right\rVert}

\newcommand{\abs}[1]{\left\lvert {#1} \right\rvert}
\DeclareMathOperator{\sign}{sign}

\newcommand{\code}[1]{\mbox{\texttt{#1}}}

\usepackage[accepted]{icml2012}

\icmltitlerunning{The Kernelized Stochastic Batch Perceptron}

\begin{document}

\twocolumn[

\icmltitle{The Kernelized Stochastic Batch Perceptron}

\icmlauthor{Andrew Cotter}{cotter@ttic.edu}
\icmladdress{Toyota Technological Institute at Chicago \
            6045 S. Kenwood Ave., Chicago, IL 60637 USA}
\icmlauthor{Shai Shalev-Shwartz}{shais@cs.huji.ac.il}
\icmladdress{John S. Cohen SL in CS, 
            The Hebrew University of Jerusalem, Israel}
\icmlauthor{Nathan Srebro}{nati@ttic.edu}
\icmladdress{Toyota Technological Institute at Chicago \
            6045 S. Kenwood Ave., Chicago, IL 60637 USA}

\icmlkeywords{machine learning, ICML}

\vskip 0.3in

]

\begin{abstract}
We present a novel approach for training kernel Support Vector
Machines, establish learning runtime guarantees for our method that
are better then those of any other known kernelized SVM optimization
approach, and show that our method works well in practice compared
to existing alternatives.
\end{abstract}

\section{Introduction}\label{sec:introduction}

We present a novel algorithm for training kernel Support Vector Machines
(SVMs).
%
%
One may view a SVM as the bi-criterion optimization problem of seeking a
predictor with large margin (low norm) on the one hand, and small training
error on the other. Our approach is a stochastic gradient
method on a non-standard scalarization of this bi-criterion problem. In
particular, we use the ``slack constrained'' scalarized optimization problem
introduced by \citet{HazanKoSr11} where we seek to maximize the classification
margin, subject to a constraint on the total amount of ``slack'', i.e.~sum of
the violations of this margin. Our approach is based on an efficient method for
computing unbiased gradient estimates on the objective. Our algorithm can be
seen as a generalization of the ``Batch Perceptron'' to the non-separable case
(i.e.~when errors are allowed), made possible by introducing stochasticity, and
we therefore refer to it as the ``Stochastic Batch Perceptron'' (SBP).
%
%

The SBP is fundamentally different from Pegasos \cite{ShalevSiSrCo10} and other
stochastic gradient approaches to the problem of training SVMs, in that calculating
each stochastic gradient estimate still requires considering the {\em entire}
data set. In this regard, despite its stochasticity, the SBP is very much a
``batch'' rather than ``online'' algorithm. For a linear SVM, each iteration
would require runtime linear in the training set size, resulting in an
unacceptable overall runtime. However, in the kernel setting, essentially all
known approaches already require linear runtime per iteration. A more careful
analysis reveals the benefits of the SBP over previous kernel SVM optimization
algorithms.
%
%

In order to compare the SBP runtime to the runtime of other SVM optimization
algorithms, which typically work on different scalarizations of the
bi-criterion problem, we follow \citet{BottouBo08,ShalevSr08} and compare the
runtimes required to ensure a generalization error of $\mathcal{L}^*+\epsilon$,
assuming the existence of some unknown predictor $u$ with norm $\norm{u}$ and
expected hinge loss $\mathcal{L}^*$. The main advantage of the SBP is in the
regime in which $\epsilon = \Omega(\mathcal{L}^*)$, i.e.~we seek a constant
factor approximation to the best achievable error (e.g.~we would like an error
of $1.01 \mathcal{L}^*$). In this regime, the overall SBP runtime is
$\norm{u}^4/\epsilon$, compared with $\norm{u}^4/\epsilon^3$ for Pegasos and
$\norm{u}^4/\epsilon^2$ for the best known dual decomposition approach.

\section{Setup and Formulations}

Training a SVM amounts to finding a vector $w$ defining a classifier $x
\mapsto \sign(\inner{w}{\Phi\left(x\right)})$, that on the one hand has small
norm (corresponding to a large classification margin), and on the other has a
small training error, as measured through the average hinge loss on the
training sample: $\hat{\mathcal{L}}(w) = \frac{1}{n}\sum_{i=1}^{n}\ell\left(y_i
\inner{w}{\Phi\left(x_i\right)}\right)$, where each $\left(x_{i},y_{i}\right)$
is a labeled example, and $\ell\left(a\right)=\max\left(0,1-a\right)$ is the
hinge loss. This is captured by the following bi-criterion optimization
problem:
\begin{equation}
\label{eq:bi-criterion-objective} \min_{w\in\mathbb{R}^{d}} \;\; \norm{w}
\;\;\; , \;\;\; \hat{\mathcal{L}}(w) .
\end{equation}
We focus on \emph{kernelized} SVMs, where the feature map $\Phi(x)$ is
specified implicitly via a kernel $K\left(x,x'\right) =
\inner{\Phi\left(x\right)}{\Phi\left(x'\right)}$, and assume that $K(x,x')\leq
1$. We consider only ``black box'' access to the kernel (i.e.~our methods work
for any kernel, as long as we can compute $K(x,x')$ efficiently), and in our
runtime analysis treat kernel evaluations as requiring $O(1)$ runtime.  Since
kernel evaluations dominate the runtime of all methods studied (ours as well as
previous methods), one can also interpret the runtimes as indicating the number
of required kernel evaluations. To simplify our derivation, we often discuss
the explicit SVM, using $\Phi(x)$, and refer to the kernel only when needed.

A typical approach to the bi-criterion Problem
\ref{eq:bi-criterion-objective} is to scalarize it using a parameter
$\lambda$ controlling the tradeoff between the norm (inverse margin)
and the empirical error:
\begin{equation}
\label{eq:regularized-objective} \min_{w\in\mathbb{R}^{d}}\frac{\lambda}{2}
\norm{w}^{2} +\frac{1}{n}\sum_{i=1}^{n}\ell\left(y_i\inner{w}{\Phi\left(x_{i}\right)}\right)
\end{equation}
Different values of $\lambda$ correspond to different Pareto optimal solutions
of Problem \ref{eq:bi-criterion-objective}, and the entire Pareto front can be
explored by varying $\lambda$.

We instead consider the ``slack constrained'' scalarization \cite{HazanKoSr11},
where we maximize the ``margin'' subject to a constraint of $\nu$ on the total
allowed ``slack'', corresponding to the average error. That is, we aim at
maximizing the margin by which all points are correctly classified (i.e.~the
minimal distance between a point and the separating hyperplane), after allowing
predictions to be corrected by a total amount specified by the slack
constraint:
\begin{align}
\label{eq:slack-constrained-objective} \max_{w\in\mathbb{R}^{d}}
\max_{\xi\in\mathbb{R}^{n}} \min_{i\in\left\{1,\dots,n\right\}} & \left(
y_{i}\inner{w}{\Phi\left(x_{i}\right)}+\xi_{i} \right) \\
\nonumber \mbox{subject to: } & \norm{w} \le 1,~ \xi \succeq 0, ~ \mathbf{1}^{T}\xi\le n\nu
\end{align}
In this scalarization, varying $\nu$ explores different Pareto optimal
solutions of Problem \ref{eq:bi-criterion-objective}. This is captured by the
following Lemma, which also quantifies how suboptimal solutions of the
slack-constrained objective correspond to Pareto suboptimal points:

\medskip
\begin{lem}
\label{lem:slack-constrained-suboptimality}

\emph{\citep[Lemma 2.1]{HazanKoSr11}}
%
%
For any $u \ne 0$, consider Problem \ref{eq:slack-constrained-objective} with
$\nu = \hat{\mathcal{L}}\left(u\right)/\norm{u}$. Let $\bar{w}$ be an
$\bar{\epsilon}$-suboptimal solution to this problem with objective value
$\gamma$, and consider the rescaled solution $w=\bar{w}/\gamma$. Then:
\begin{align*}
\norm{w} \le & \frac{1}{1-\bar{\epsilon}\norm{u}}\norm{u} ~~,~~
\hat{\mathcal{L}}\left(w\right) \le  \frac{1}{1-\bar{\epsilon}\norm{u}}
\hat{\mathcal{L}}\left(u\right)
\end{align*}

\end{lem}


\section{The Stochastic Batch Perceptron}\label{sec:algorithm}

In this section, we will develop the Stochastic Batch Perceptron. We consider
Problem \ref{eq:slack-constrained-objective} as optimization of the variable
$w$ with a single constraint $\norm{w}\leq 1$, with the objective being to
maximize:
\begin{equation}
\label{eq:fw-definition} f\left(w\right) ~= \max_{\xi\succeq 0,
\mathbf{1}^{T}\xi\le n\nu} ~~\min_{p\in\Delta^{n}} ~~\sum_{i=1}^{n} p_{i}
\left( y_{i}\inner{w}{\Phi\left(x_{i}\right)}+\xi_{i} \right)
\end{equation}
Notice that we replaced the minimization over training indices $i$ in Problem
\ref{eq:slack-constrained-objective} with an equivalent minimization over the
probability simplex, $\Delta^n = \{p \succeq 0 : \mathbf{1}^T p = 1\}$, and
that we consider $p$ and $\xi$ to be a part of the objective, rather than
optimization variables. The objective $f(w)$ is a concave function of $w$, and
we are maximizing it over a convex constraint $\norm{w}\leq 1$, and so this is
a convex optimization problem in $w$.

Our approach will be to perform a stochastic gradient update on $w$ at each
iteration: take a step in the direction specified by an unbiased estimator of a
(super)gradient of $f(w)$, and project back to $\norm{w}\leq 1$. To this end,
we will need to identify the (super)gradients of $f(w)$ and understand how to
efficiently calculate unbiased estimates of them.

\subsection{Warmup: The Separable Case}\label{subsec:warmup}

As a warmup, we first consider the separable case, where $\nu=0$ and no errors
are allowed. The objective is then:
\begin{equation}
\label{eq:sep-fw} f(w) = \min_i y_{i}\inner{w}{\Phi\left(x_{i}\right)},
\end{equation}
This is simply the ``margin'' by which all points are correctly classified,
i.e.~$\gamma$ s.t.~$\forall_i~ y_i \inner{w}{\Phi(x_i)}\geq\gamma$. We seek a
linear predictor $w$ with the largest possible margin. It is easy to see that
(super)gradients with respect to $w$ are given by $y_i \Phi(x_i)$ for any index
$i$ attaining the minimum in Equation \ref{eq:sep-fw}, i.e.~by the ``most poorly
classified'' point(s). A gradient ascent approach would then be to iteratively
find such a point, update $w \leftarrow w+\eta y_i \Phi(x_i)$, and project back
to $\norm{w}\leq 1$. This is akin to a ``batch Perceptron'' update, which at
each iteration searches for a violating point and adds it to the predictor.

In the separable case, we could actually use \emph{exact} supergradients of the
objective. As we shall see, it is computationally beneficial in the
non-separable case to base our steps on unbiased gradient estimates. We
therefore refer to our method as the ``Stochastic Batch Perceptron'' (SBP), and
view it as a generalization of the batch Perceptron which uses stochasticity
and is applicable in the non-separable setting. In the same way that the
``batch Perceptron'' can be used to maximize the margin in the separable case,
the SBP can be used to obtain any SVM solution along the Pareto front of the
bi-criterion Problem \ref{eq:bi-criterion-objective}.

\subsection{Supergradients of $f(w)$}\label{subsec:minimax-optimality}

\begin{figure}[t]

\centering
\includegraphics[width=0.6\columnwidth]{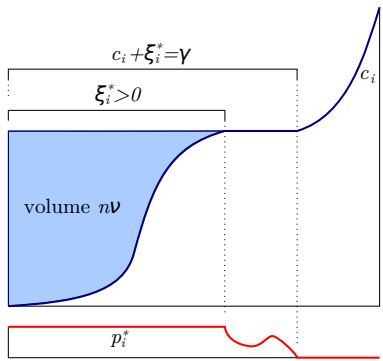}

\caption{
\small
Illustration of how one finds $\xi^{*}$ and $p^{*}$. The upper curve
represents the values of the responses $c_{i}$, listed in order of increasing
magnitude.
The lower curve illustrates a minimax optimal probability distribution $p^{*}$.
%
}

\label{fig:water}

\end{figure}

For a fixed $w$, we define $c \in \mathbb{R}^n$ be the vector of ``responses'':
\begin{equation}
\label{eq:responses-definition} c_{i} = y_{i}\inner{w}{\Phi\left(x_{i}\right)}
\end{equation}

Supergradients of $f(w)$ at $w$ can be characterized explicitly in terms of
minimax-optimal pairs $p^*$ and $\xi^*$ such that $p^* =
\arg\min_{p\in\Delta^n} p^t(c+\xi^*)$ and $\xi^* =
\arg\max_{\xi\succeq 0, \mathbf{1}^{T}\xi\le n\nu} (p^*)^T(c+\xi)$.

\medskip
\ifinappendix
	\begin{replem}{lem:slack-constrained-supergradient}
\else
	\begin{lem}[Proof in Appendix \ref{app:proofs}]\label{lem:slack-constrained-supergradient}
\fi

For any $w$, let $p^*,\xi^*$ be minimax optimal for Equation
\ref{eq:fw-definition}. Then $\sum_{i=1}^{n}
p^{*}_{i}y_{i}\Phi\left(x_{i}\right)$ is a supergradient of $f(w)$ at $w$.

\ifinappendix
	\end{replem}
\else
	\end{lem}
\fi

This suggests a simple method for obtaining unbiased estimates of
supergradients of $f(w)$: sample a training index $i$ with probability
$p^{*}_{i}$, and take the stochastic supergradient to be
$y_{i}\Phi\left(x_{i}\right)$. The only remaining question is how one
finds a minimax optimal $p^{*}$.

It is possible to find a minimax optimal $p^*$ in $O(n)$ time. For any $\xi$, a
solution of $\min_{p \in \Delta^n} p^T(x+\xi)$ must put all of the probability
mass on those indices $i$ for which $c_{i}+\xi_{i}$ is minimized. Hence, an
optimal $\xi^{*}$ will maximize the minimal value of $c_{i}+\xi^{*}_{i}$. This
is illustrated in Figure \ref{fig:water}. The intuition is that the total mass
$n\nu$ available to $\xi$ is distributed among the indices as if this volume of
water were poured into a basin with height $c_{i}$. The result is that the
indices $i$ with the lowest responses have columns of water above them such
that the common surface level of the water is $\gamma$.

Once the ``water level'' $\gamma$ has been determined, the optimal $p^*$ must
be uniform on those indices $i$ for which $\xi^*_i>0$, i.e.~for which
$c_i<\gamma$, must be zero on all $i$ s.t.~$c_i>\gamma$, and could take any
intermediate value when $c_i=\gamma$ (that is, for some $q>0$, we must have
$c_i<\gamma \rightarrow p^*_i=q$, $c_i=\gamma \rightarrow 0\leq p^*_i \leq q$,
and $c_i>\gamma \rightarrow p^*_i=0$---see Figure \ref{fig:water}). In
particular, the uniform distribution over all indices such that $c_{i}\leq
\gamma$ is minimax optimal. Notice that in the separable case, where no slack
is allowed, $\gamma=\min_i c_i$ and any distribution supported on the
minimizing point(s) is minimax optimal, and $y_i \Phi(x_i)$ is an {\em exact}
supergradient for such an $i$, as discussed in Section \ref{subsec:warmup}.

It is straightforward to find the water level $\gamma$ in linear time once the
responses $c_i$ are sorted (as in Figure \ref{fig:water}), i.e.~with a total
runtime of $O(n \log n)$ due to sorting. It is also possible to find the water
level $\gamma$ in linear time, without sorting the responses, using a
divide-and-conquer algorithm, further of which may be found in Appendix
\ref{app:algorithm}.
%
%

\subsection{Kernelized Implementation}\label{subsec:slack-constrained-kernel}

In a kernelized SVM, $w$ is an element of an implicit space, and cannot be
represented explicitly. We therefore represent $w$ as $w = \sum_{i=1}^{n}
\alpha_{i}y_{i}\Phi\left(x_{i}\right)$, and maintain not $w$ itself, but
instead the coefficients $\alpha_i$. Our stochastic gradient estimates are
always of the form $y_i \Phi(x_i)$ for an index $i$. Taking a step in
this direction amounts to simply increasing the corresponding $\alpha_i$.

We could calculate all the responses $c_i$ at each iteration as $c_i =
\sum_{j=1}^{n} \alpha_j y_i y_j K(x_i,x_j)$. However, this would require a
quadratic number of kernel evaluations {\em per iteration}. Instead, as is
typically done in kernelized SVM implementations, we keep the responses $c_i$
on hand, and after each stochastic gradient step of the form $w \leftarrow
w+\eta y_j \Phi\left( x_j \right)$,
we update the responses as:
\begin{equation}
\label{eq:ciupdate} c_i \leftarrow c_i + \eta y_i y_j K(x_i,x_j)
\end{equation}
This involves only $n$ kernel evaluations per iteration.

In order to project $w$ onto the unit ball, we must either track $\norm{w}$ or
calculate it from the responses as $\norm{w}=\sum_{i=1}^n \alpha_i c_i$.
Rescaling $w$ so as to project it back into $\norm{w}\leq 1$ is performed by
rescaling all coefficients $\alpha_i$ and responses $c_i$, again taking time
$O(n)$ and no additional kernel evaluations.

\subsection{Putting it Together}\label{subsec:slack-constrained-sgd}

We are now ready to summarize the SBP algorithm. Starting from $w^{(0)}=0$ (so
both $\alpha^{(0)}$ and all responses are zero), each iteration proceeds as
follows:
\vspace{-0.5em}
\begin{enumerate*}
\item Find $p^*$ by finding the ``water level'' $\gamma$ from the responses
(Section \ref{subsec:minimax-optimality}), and taking $p^*$ to be uniform on
those indices for which $c_i\leq \gamma$.
\item Sample $j \sim p^*$.
\item Update $w^{(t+1)} \leftarrow
\mathcal{P}\left(w^{(t)}+\eta_{t}y_{j}\Phi\left(x_{j}\right)\right)$, where
$\mathcal{P}$ projects onto the unit ball and $\eta_{t} = \frac{1}{\sqrt{t}}$.
This is done by first increasing $\alpha \leftarrow\alpha + \eta_{t}$ and
updating the responses as in Equation \ref{eq:ciupdate}, then calculating
$\norm{w}$ (Section \ref{subsec:slack-constrained-kernel}) and scaling $\alpha$
and $c$ by $\min(1,1/\norm{w})$.
\end{enumerate*}
\vspace{-0.5em}
Updating the responses as in Equation \ref{eq:ciupdate} requires $O(n)$ kernel
evaluations (the most computationally expensive part) and all other operations
require $O(n)$ scalar arithmetic operations.


Since at each iteration we are just updating using an unbiased
estimator of a supergradient, we can rely on the standard analysis of
stochastic gradient descent to bound the suboptimality after $T$
iterations: 

\medskip 
\ifinappendix
	\begin{replem}{lem:zinkevich-batch}
\else
	\begin{lem}[Proof in Appendix \ref{app:proofs}]\label{lem:zinkevich-batch}
\fi

For any $T,\delta>0$, after $T$ iterations of the Stochastic Batch
Perceptron, with probability at least $1-\delta$, the average iterate
$\bar{w} = \frac{1}{T}\sum_{t=1}^{T} w^{(t)}$ (corresponding to
$\bar{\alpha} = \frac{1}{T} \sum_{t=1}^{T} \alpha^{(t)}$), satisfies:
%
%
$
f\left(\bar{w}\right) \geq 
\sup_{\norm{w}\leq 1} f\left(w \right) -
O\left(\sqrt{\frac{1}{T}\log\frac{1}{\delta}}\right).
$
%

\ifinappendix
	\end{replem}
\else
	\end{lem}
\fi

Since each iteration is dominated by $n$ kernel evaluations, and thus takes
linear time (we take a kernel evaluation to require $O(1)$ time), the overall
runtime to achieve $\epsilon$ suboptimality for Problem
\ref{eq:slack-constrained-objective} is $O(n/\epsilon^2)$.

\subsection{Learning Runtime}\label{subsec:slack-constrained-runtime}

The previous section has given us the runtime for obtaining a certain
suboptimality of Problem \ref{eq:slack-constrained-objective}. However, since
the suboptimality in this objective is not directly comparable to the
suboptimality of other scalarizations, e.g. Problem
\ref{eq:regularized-objective}, we follow \citet{BottouBo08,ShalevSr08}, and
analyze the runtime required to achieve a desired generalization performance,
instead of that to achieve a certain optimization accuracy on the empirical
optimization problem.

\newcommand{\genR}{\mathcal{L}_{0/1}}
\newcommand{\genRh}{\mathcal{L}}
\newcommand{\gendelta}{\epsilon}

Recall that our true learning objective is to find a predictor with low
generalization error $\genR(w) = \probability[(x,y)]{ y\inner{w}{\Phi(x)} \leq
0 }$ with respect to some unknown distribution over $x,y$ based on a training
set drawn i.i.d.~from this distribution. We assume that there exists some
(unknown) predictor $u$ that has norm $\norm{u}$ and low expected hinge loss
$\mathcal{L}^* = \genRh(u) = \expectation{\ell(y\inner{u}{\Phi(x)})}$
(otherwise, there is no point in training a SVM), and analyze the runtime to
find a predictor $w$ with generalization error $\genR(w) \leq
\mathcal{L}^*+\gendelta$.

In order to understand the SBP runtime, we must determine both the required
sample size and optimization accuracy. Following \citet{HazanKoSr11}, and based
on the generalization guarantees of \citet{SrebroSrTe10}, using a sample of
size:
\begin{equation}
\label{eq:samplesize} n = \tilde{O}\left(
\left(\frac{\mathcal{L}^*+\gendelta}{\gendelta}\right)
\frac{\norm{u}^2}{\gendelta} \right)
\end{equation}
and optimizing the empirical SVM bi-criterion Problem
\ref{eq:bi-criterion-objective} such that:
\begin{align}
\label{eq:slack-constrained-target} \norm{w} & \le 2\norm{u} ~~;~~ 
 \hat{\mathcal{L}}\left(w\right) - \hat{\mathcal{L}}\left(u\right) \le
\gendelta/2
\end{align}
suffices to ensure $\genR(w) \leq \mathcal{L}^*+\gendelta$ with high
probability.
Referring to Lemma \ref{lem:slack-constrained-suboptimality}, Equation
\ref{eq:slack-constrained-target} will be satisfied for $\bar{w}/\gamma$ as
long as $\bar{w}$ optimizes the objective of Problem
\ref{eq:slack-constrained-objective} to within:
\begin{equation}
\label{eq:reqbarepsilon} \bar{\epsilon} = \frac{\gendelta/2}{\norm{u}
(\hat{\mathcal{L}}(u)+\gendelta/2)} \geq \Omega\left(\frac{\gendelta}{\norm{u}
(\hat{\mathcal{L}}(u)+\gendelta)}\right)
\end{equation}
where the inequality holds with high probability for the sample size of
Equation \ref{eq:samplesize}. Plugging this sample size and the optimization
accuracy of Equation \ref{eq:reqbarepsilon} into the SBP runtime of
$O(n/\bar{\epsilon}^2)$ yields the overall runtime:
\begin{equation}
\label{eq:SBP-learning-runtime} \tilde{O}\left(
\left(\frac{\mathcal{L}^*+\epsilon}{\epsilon}\right)^3
\frac{\norm{u}^4}{\epsilon} \right)
\end{equation}
for the SBP to find $\bar{w}$ such that its rescaling satisfies $\genR(w) \leq
\genRh(u) + \gendelta$ with high probability.

In the realizable case, where $\mathcal{L}^*=0$, or more generally when we would
like to reach $\mathcal{L}^*$ to within a small constant multiplicative factor,
we have $\epsilon=\Omega(\mathcal{L}^*)$, the first factor in Equation
\ref{eq:SBP-learning-runtime} is a constant, and the runtime
simplifies to $\tilde{O}(\norm{u}^4/\gendelta)$. As we will see in
Section \ref{sec:comparison}, this is a better guarantee than that
enjoyed by any other SVM optimization approach.

\subsection{Including an Unregularized Bias}\label{subsec:unregularized-bias}

It is possible to use the SBP to train SVMs with a bias term, i.e.~where
one seeks a predictor of the form $x \mapsto (\inner{w}{\Phi(x)}+b)$.
We then take stochastic gradient steps on:
\begin{align}
\MoveEqLeft \label{eq:fw-definition-bias} f(w) = \\
\notag & \max_{\begin{array}{c} \scriptstyle b\in\mathbb{R}, \xi\succeq 0 \\
\scriptstyle \mathbf{1}^{T}\xi\le n\nu \end{array}}\min_{p\in\Delta^{n}}
\sum_{i=1}^{n} p_{i} \left( y_{i}\inner{w}{\Phi(x_{i})} + y_{i} b + \xi_{i}
\right)
\end{align}
Lemma \ref{lem:slack-constrained-supergradient} still holds, but we must now
find minimax optimal $p^*$,$\xi^*$ and $b^*$. This can be accomplished using a
modified ``water filling'' involving two basins, one containing the
positively-classified examples, and the other the negatively-classified ones.
As in the case without an unregularized bias, this can be accomplished in
$O(n)$ time---see Appendix \ref{app:algorithm} for details.


\section{Relationship to Other Methods}\label{sec:comparison}

\begin{table}[*t]

\caption{
\small
Upper bounds, up to log factors, on the runtime (number of kernel evaluations)
required to achieve $\genR(w)\leq \genRh(u)+\epsilon$.
}

\begin{center}
\begin{tabular}{l|cc}
\hline
\abovespace\belowspace
 & Overall & $\epsilon=\Omega\left(\mathcal{L}\left(u\right)\right)$ \\
\hline
\abovespace
SBP & $\left(\frac{\mathcal{L}\left(u\right)+\epsilon}{\epsilon}\right)^{3}\frac{\norm{u}^{4}}{\epsilon}$ & $\frac{\norm{u}^{4}}{\epsilon}$ \\
Dual Decomp. & $\left(\frac{\mathcal{L}\left(u\right)+\epsilon}{\epsilon}\right)^{2}\frac{\norm{u}^{4}}{\epsilon^{2}}$ & $\frac{\norm{u}^{4}}{\epsilon^{2}}$ \\
\belowspace
SGD on $\hat{\mathcal{L}}$ & $\left(\frac{\mathcal{L}\left(u\right)+\epsilon}{\epsilon}\right)\frac{\norm{u}^{4}}{\epsilon^{3}}$ & $\frac{\norm{u}^{4}}{\epsilon^{3}}$ \\
\hline
\end{tabular}
\end{center}

\label{tab:bounds}

\end{table}

We discuss the relationship between the SBP and several other SVM
optimization approaches, highlighting similarities and key
differences, and comparing their performance guarantees.

\subsection{SIMBA}

Recently, \citet{HazanKoSr11} presented SIMBA, a method for training {\em
linear} SVMs based on the same ``slack constrained'' scalarization (Problem
\ref{eq:slack-constrained-objective}) we use here. SIMBA also fully optimizes
over the slack variables $\xi$ at each iteration, but differs in that, instead
of fully optimizing over the distribution $p$ (as the SBP does), SIMBA updates
$p$ using a stochastic mirror descent step. The predictor $w$ is then updated,
as in the SBP, using a random example drawn according to $p$. A SBP iteration
is thus in a sense more ``thorough'' then a SIMBA iteration. The SBP
theoretical guarantee
(Lemma \ref{lem:zinkevich-batch}) is correspondingly better by a logarithmic
factor (compare to \citet[Theorem 4.3]{HazanKoSr11}). All else being equal, we
would prefer performing a SBP iteration over a SIMBA iteration.

For linear SVMs,
a SIMBA iteration can be performed in time $O(n+d)$. However, fully optimizing
$p$ as described in Section \ref{subsec:minimax-optimality} requires the
responses $c_i$, and calculating or updating all $n$ responses would
require time $O(nd)$. In this setting, therefore, a SIMBA iteration is
much more efficient than a SBP iteration.

In the kernel setting, calculating even a single response requires $O(n)$
kernel evaluation, which is the same cost as updating \emph{all} responses
after a change to a single coordinate $\alpha_{i}$ (Section
\ref{subsec:slack-constrained-kernel}). This makes the responses essentially
``free'', and gives an advantage to methods such as the SBP (and the dual
decomposition methods discussed below) which make use of the responses.

Although SIMBA is preferable for linear SVMs,
the SBP is preferable for kernelized SVMs.
It should also be noted that SIMBA relies heavily on having direct access to
features, and that it is therefore not obvious how to apply it directly in the
kernel setting.

\subsection{Pegasos and SGD on $\hat{\mathcal{L}}(w)$}

Pegasos \cite{ShalevSiSrCo10} is a SGD method optimizing the regularized
scalarization of Problem \ref{eq:regularized-objective}. Alternatively, one can
perform SGD on $\hat{\mathcal{L}}(w)$ subject to the constraint that
$\norm{w}\le B$, yielding similar learning guarantees (e.g. \cite{Zhang04}). At
each iteration, these algorithms pick an example uniformly at random from the
training set. If the margin constraint is violated on the example, $w$ is
updated by adding to it a scaled version of $y_i \Phi(x_i)$.
Then, $w$ is scaled and possibly projected back to $\norm{w}\leq B$. The actual
update performed at each iteration is thus very similar to that of the SBP.
The main difference is that in Pegasos and related SGD approaches, examples are
picked uniformly at random, unlike the SBP which samples from the set of
violating examples.

In a linear SVM, where $\Phi(x_i)\in\mathbb{R}^d$ are given explicitly, each
Pegasos (or SGD on $\hat{\mathcal{L}}(w)$) iteration is extremely simple and
requires runtime which is linear in the dimensionality of $\Phi(x_i)$. A SBP
update would require calculating and referring to all $O(n)$ responses.
However, with access only to kernel evaluations, even a Pegasos-type update
requires either considering all support vectors, or alternatively updating all
responses, and might also take $O(n)$ time, just like the much ``smarter'' SBP
step.

To understand the learning runtime of such methods in the kernel setting,
recall that SGD converges to an $\epsilon$-accurate solution of the
optimization problem after at most $\norm{u}^2/\epsilon^2$ iterations.
Therefore, the overall runtime is $n\norm{u}^2/\epsilon^2$. Combining this with
Equation \ref{eq:samplesize} yields that the runtime requires by SGD to achieve
a learning accuracy of $\epsilon$ is $ \tilde{O}\left(
((\mathcal{L}^*+\gendelta)/\gendelta) \norm{u}^4/\gendelta^3 \right) $.  When
$\epsilon = \Omega(\mathcal{L}^*)$, this scales as $1/\epsilon^3$ compared with
the $1/\epsilon$ scaling for the SBP (see also Table \ref{tab:bounds}).

\subsection{Dual Decomposition Methods}

Many of the most popular packages for optimizing kernel SVMs, including LIBSVM
\citep{ChangLi01} and SVM-Light \citep{Joachims98}, use dual-decomposition
approaches. This family of algorithms works on the dual of the scalarization
\ref{eq:regularized-objective}, given by:
\begin{equation}
\label{eq:dual}
\max_{\alpha \in \left[0,\frac{1}{\lambda\,n}\right]^n} ~~ \sum_{i=1}^n
\alpha_i - \frac{1}{2} \sum_{i,j=1}^n \alpha_i \alpha_j y_i y_j K(x_i,x_j)
\end{equation}
and proceed by iteratively choosing a small working set of dual variables
$\alpha_{i}$, and then optimizing over these variables while holding all other
dual variables fixed. At an extreme, SMO \citep{Platt98} uses a working set of
the smallest possible size (two in problems with an unregularized bias, one in
problems without).
Most dual decomposition approaches rely on having access to all the responses
$c_i$ (as in the SBP), and employ some heuristic to select variables
$\alpha_{i}$ that are likely to enable a significant increase in the dual
objective.

On an objective without an unregularized bias the structure of SMO is similar
to the SBP: the responses $c_i$ are used to choose a single point $j$ in the
training set, then $\alpha_j$ is updated,
and finally the responses are updated accordingly. There are two important
differences, though: how the training example to update is chosen, and how the
change in $\alpha_j$ is performed.

SMO updates $\alpha_j$ so as to exactly optimize the \emph{dual} Problem
\ref{eq:dual}, while the SBP takes a step along $\alpha_j$ so as to improve the
{\em primal} Problem \ref{eq:slack-constrained-objective}. Dual feasibility is
{\em not} maintained, so the SBP has more freedom to use large coefficients on
a few support vectors, potentially resulting in sparser solutions.

The use of heuristics to choose the training example to update makes SMO very
difficult to analyze.  Although it is known to converge linearly after some
number of iterations \cite{chen2006study}, the number of iterations required to
reach this phase can be very large (see a detailed discussion in Appendix
\ref{app:dual-decomposition}). To the best of our knowledge,
the most satisfying analysis for a dual decomposition method is the one given
in \citet{hush2006qp}.
In terms of learning runtime, this analysis yields a runtime of
$\tilde{O}\left(\left(\left(\genRh(u)+\epsilon\right)/\epsilon\right)^2\norm{u}^4/\epsilon^2\right)$
to guarantee $\genR(w) \leq \genRh(u)+\epsilon$. When $\epsilon=\Omega(L^*)$,
this runtime scales as $1/\epsilon^2$, compared with the $1/\epsilon$ guarantee
for the SBP.

\subsection{Stochastic Dual Coordinate Ascent}

Another variant of the dual decomposition approach is to choose a single
$\alpha_i$ randomly at each iteration and update it so as to optimize Equation
\ref{eq:dual} \cite{HsiehChLiKeSu08}. The advantage here is that we do not
need to use all of the responses at each iteration, so that if it is easy to
calculate responses on-demand, as in the case of linear SVMs,
each SDCA iteration can be calculated in time $O(d)$ \cite{HsiehChLiKeSu08}.
In a sense, SDCA relates to SMO in a similar fashion that Pegasos relates to
the SBP: SDCA and Pegasos are preferable on linear SVMs since they choose
working points at random; SMO and the SBP choose working points based on more
information (namely, the responses), which are unnecessarily expensive to
compute in the linear case, but, as discussed earlier, are essentially ``free''
in kernelized implementations. Pegasos and the SBP both work on the primal
(though on different scalarizations), while SMO and SDCA work on the dual and
maintain dual feasibility.

The current best analysis of the runtime of SDCA is not satisfying, and yields
the bound $n/\lambda \epsilon$ on the number of iterations, which is a factor
of $n$ larger than the bound for Pegasos. Since the cost of each iteration is
the same, this yields a significantly worse guarantee. We do not know if a
better guarantee can be derived for SDCA. See a detailed discussion in Appendix
\ref{app:dual-decomposition}.

\subsection{The Online Perceptron}\label{subsec:comparison-perceptron}

%
We have so far considered only the problem of optimizing the bi-criterion SVM
objective of Problem \ref{eq:bi-criterion-objective}. However, because the
online Perceptron achieves the same form of learning guarantee (despite not
optimizing the bi-criterion objective), it is reasonable to consider it, as well.

The online Perceptron makes a \emph{single} pass over the training set. At each
iteration, if $w$ errs on the point under consideration (i.e.~$y_i
\inner{w}{\Phi(x_i)}\leq 0$), then $y_i \Phi(x_i)$ is added into $w$.
Let $M$ be the number of mistakes made by the Perceptron on the sequence of
examples. Support vectors are added only when a mistake is made, and so each
iteration of the Perceptron involves at most $M$ kernel evaluations. The total
runtime is therefore $Mn$.

While the Perceptron is an online learning algorithm, it can also be
used for obtaining guarantees on the generalization error using an
online-to-batch conversion (e.g. \cite{CesaCoGe01}). 

From a bound on the number of mistakes $M$ (e.g. \citet[Corollary
5]{Shalev07}), it is possible to show that the expected number of mistakes the
Perceptron makes is upper bounded by $n\mathcal{L}(u) + \norm{u} \sqrt{n
\mathcal{L}(u)} + \norm{u}^2$. This implies that the total runtime required by
the Perceptron to achieve $\mathcal{L}_{0/1}(w) \le \mathcal{L}(u) + \epsilon$
is
$O\left( \left(\left(\mathcal{L}(u)+\epsilon\right)/\epsilon\right)^3
\norm{u}^4/\epsilon \right)$.
This is of the same order as the bound we have derived for SBP. However, the
Perceptron does {\em not} converge to a Pareto optimal solution to the
bi-criterion Problem \ref{eq:bi-criterion-objective}, and therefore cannot be
considered a SVM optimization procedure.
Furthermore, the online Perceptron generalization analysis relies on an
``online-to-batch'' conversion technique (e.g. \cite{CesaCoGe01}), and is
therefore valid only for a \emph{single} pass over the data. If we attempt to
run the Perceptron for multiple passes, then it might begin to overfit
uncontrollably. Although the worst-case theoretical guarantee obtained after a
single pass is indeed similar to that for an optimum of the SVM objective, in
practice an optimum of the empirical SVM optimization problem does seem to have
significantly better generalization performance.

\section{Experiments}\label{sec:experiments}

\begin{table*}[*t]

\caption{
\small
Datasets, downloaded from \url{http://leon.bottou.org/projects/lasvm}, and
parameters used in the experiments. In our experiments, we used the Gaussian
kernel with bandwidth $\sigma$.
}
\vspace{-1em}

\begin{center}
\begin{tabular}{lcc|ccc|ccc}
\hline
\abovespace\belowspace
& & & \multicolumn{3}{c|}{Without unreg. bias} & \multicolumn{3}{c}{With unreg. bias} \\
Data set & Training size $n$ & Testing size & $\sigma^2$ & $\lambda$ & $\nu$ & $\sigma^2$ & $\lambda$ & $\nu$ \\
\hline
\abovespace
Reuters money\_fx    & $7770$   & $3229$  & $0.5$ & $\nicefrac{1}{n}$  & $6.34\times{10}^{-4}$ &        &                        &                        \\
Adult                & $31562$  & $16282$ & $10$  & $\nicefrac{1}{n}$  & $1.10\times{10}^{-2}$ & $100$  & $\nicefrac{1}{100n}$   & $5.79\times{10}^{-4}$  \\
MNIST ``8'' vs. rest & $60000$  & $10000$ & $25$  & $\nicefrac{1}{n}$  & $2.21\times{10}^{-4}$ & $25$   & $\nicefrac{1}{1000n}$  & $6.42\times{10}^{-11}$ \\
\belowspace
Forest               & $522910$ & $58102$ &       &                    &                       & $5000$ & $\nicefrac{1}{10000n}$ & $7.62\times{10}^{-10}$ \\
\hline
\end{tabular}
\end{center}

\label{tab:datasets}

\end{table*}

We compared the SBP to other SVM optimization approaches on the datasets in
Table \ref{tab:datasets}. We compared to Pegasos \citep{ShalevSiSrCo10}, SDCA
\citep{HsiehChLiKeSu08}, and SMO \citep{Platt98} with a second order heuristic
for working point selection \cite{FanChLi05}. These approaches work on the
regularized formulation of Problem \ref{eq:regularized-objective} or its dual
(Problem \ref{eq:dual}). To enable comparison, the parameter $\nu$ for the SBP
was derived from $\lambda$ as
$\norm{\hat{w}^{*}}\nu=\frac{1}{n}\sum_{i=1}^{n}\ell\left(y_{i}\inner{w^{*}}{\Phi\left(x_{i}\right)}\right)$,
where $\hat{w}^*$ is the known (to us) optimum. 

\begin{figure*}[*t]

\centering
\begin{tabular}{ @{} L @{} S @{} S @{} S @{} }
& \large{Reuters} & \large{Adult} & \large{MNIST} \\
\rotatebox{90}{\scriptsize{Test error}} &
\includegraphics[width=0.30\textwidth]{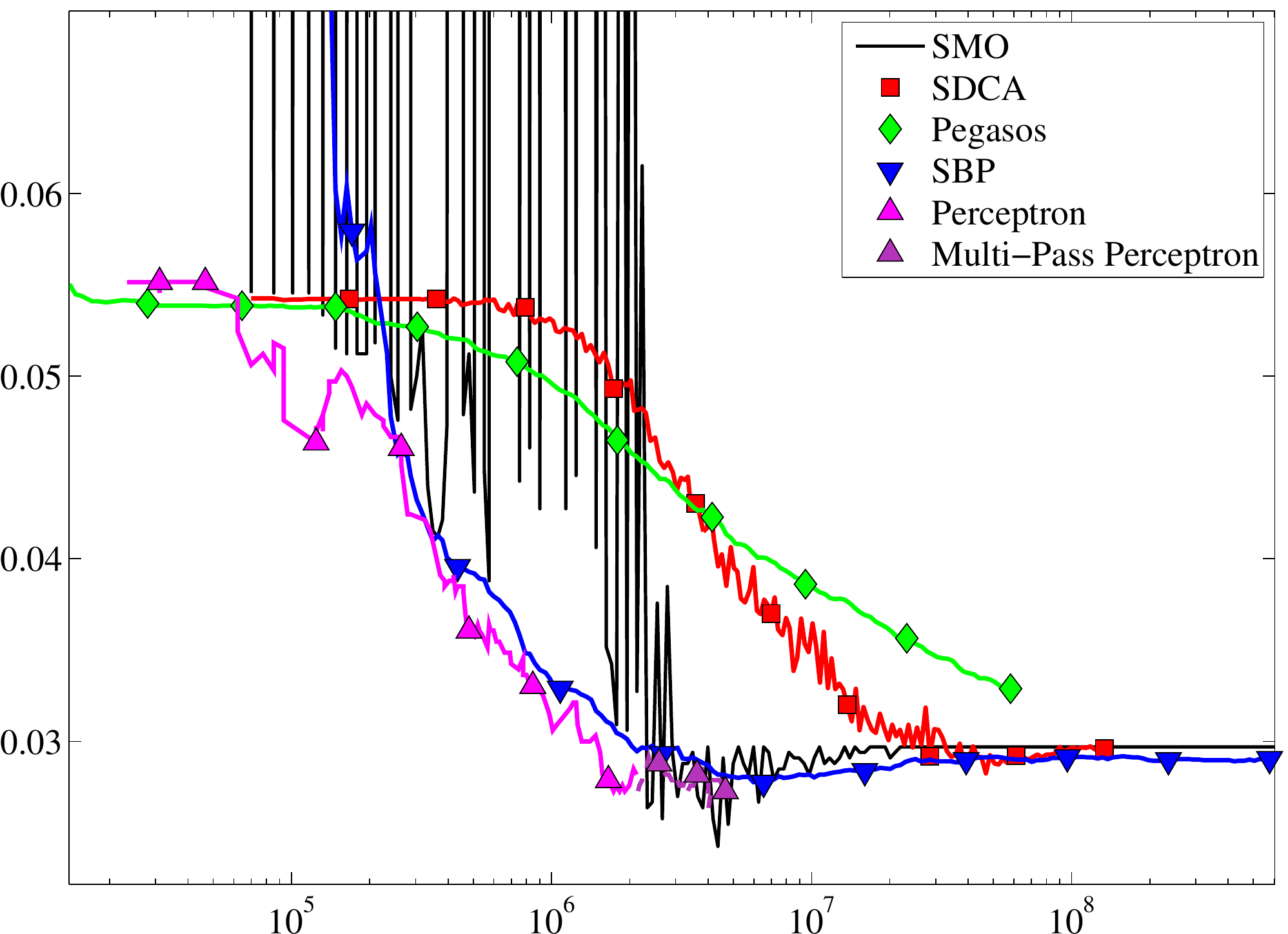} &
\includegraphics[width=0.30\textwidth]{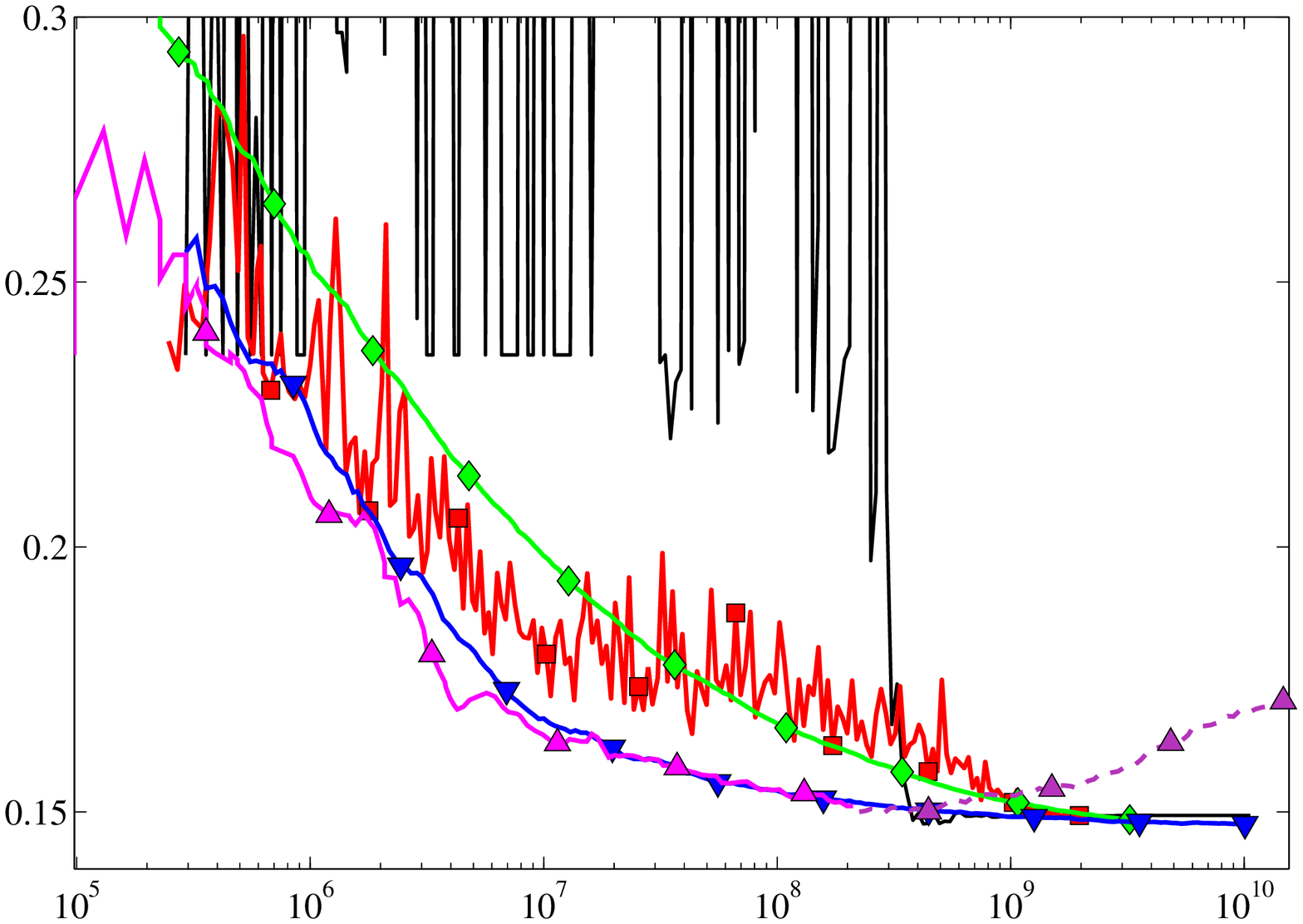} &
\includegraphics[width=0.30\textwidth]{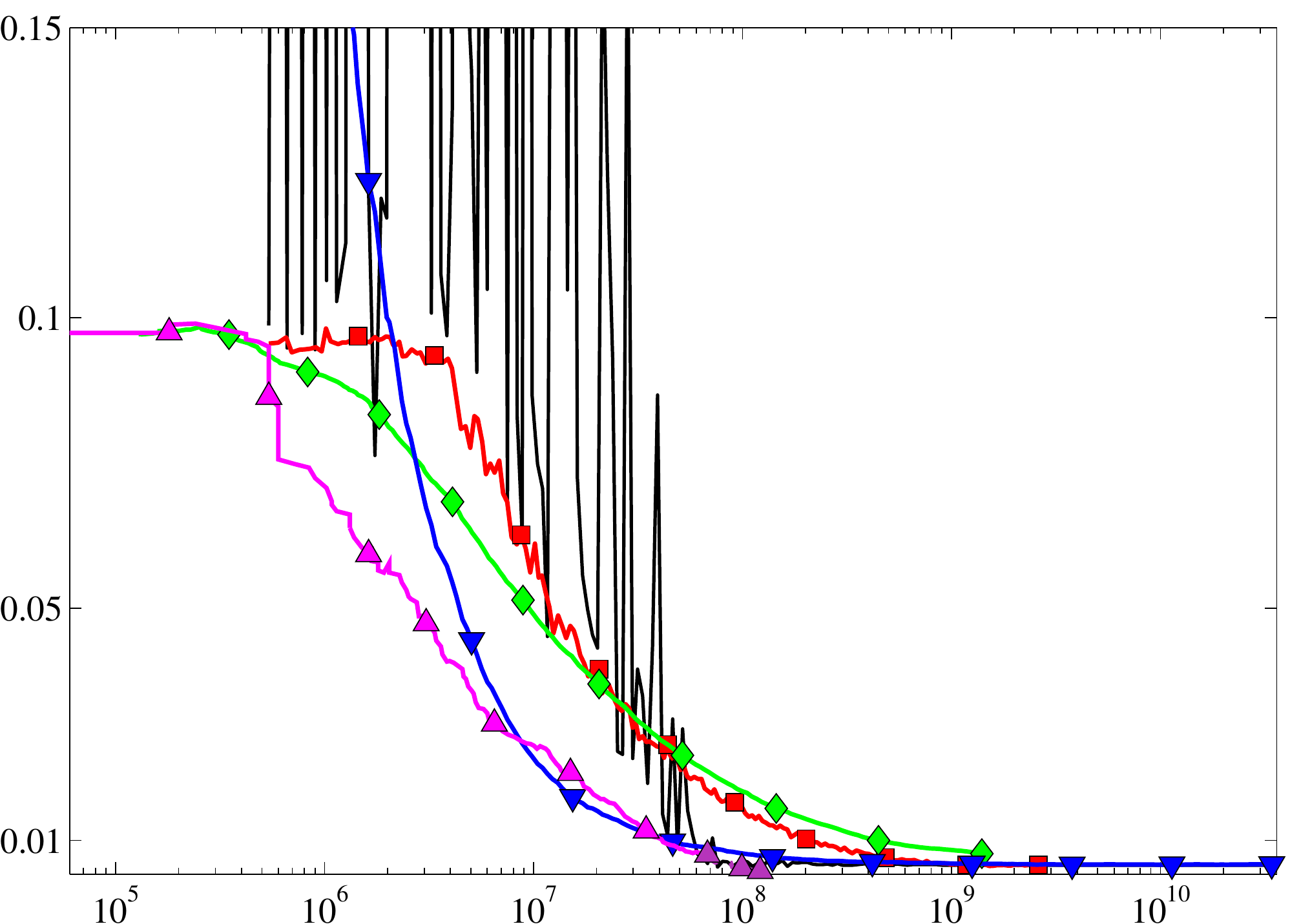} \\
& \scriptsize{Kernel Evaluations} & \scriptsize{Kernel Evaluations} & \scriptsize{Kernel Evaluations} \\
\end{tabular}

\caption{
\small
Classification error on the held-out testing set (linear scale) vs. the number
of kernel evaluations performed during optimization (log scale), averaged over
ten runs. The Perceptron was run for multiple passes over the data---its curve
becomes dashed after the first epoch ($n$ iterations). All algorithms were run
for ten epochs, \emph{except} for Perceptron on Adult, which we ran for $100$
epochs to better illustrate its overfitting.
}

\label{fig:experiments}

\end{figure*}

We first compared the methods on a SVM formulation \emph{without} an
unregularized bias, since Pegasos and SDCA do not naturally handle one. So that
this comparison would be implementation-independent, we measure performance in
terms of the number of kernel evaluations. As can be seen in Figure
\ref{fig:experiments}, the SBP outperforms Pegasos and SDCA, as predicted by
the upper bounds. The SMO algorithm has a dramatically different performance
profile, in line with the known analysis: it makes relatively little progress,
in terms of generalization error, until it reaches a certain critical point,
after which it converges rapidly. Unlike the other methods, terminating SMO
early in order to obtain a cruder solution does not appear to be advisable.

We also compared to the online Perceptron algorithm. Although use of the
Perceptron is justified for non-separable data only if run for a single pass
over the training set, we did continue running for multiple passes.
The Perceptron's generalization performance is similar to that of the SBP for
the first epoch, but the SBP continues improving over additional passes. As
discussed in Section \ref{subsec:comparison-perceptron}, the Perceptron is
unsafe and might overfit after the first epoch, an effect which is clearly
visible on the Adult dataset.

To give a sense of actual runtime, we compared our implementation of the
SBP\footnote{Source code is available from
\url{http://ttic.uchicago.edu/~cotter/projects/SBP}} to the SVM package LIBSVM,
running on an Intel E7500 processor. We allowed an unregularized bias (since
that is what LIBSVM uses), and used the parameters in Table \ref{tab:datasets}.
For these experiments, we replaced the Reuters dataset with the version of the
Forest dataset used by \citet{NguyenMaTaHa10}, using their parameters. LIBSVM
converged to a solution with $14.9$\% error in $195$s on Adult, $0.44$\% in
$1980$s on MNIST, and $1.8$\% in $35$ hours on Forest. In \emph{one-quarter} of
each of these runtimes, SBP obtained $15.0$\% error on Adult, $0.46$\% on
MNIST, and $1.6$\% on Forest. These results of course depend heavily on the
specific stopping criterion used.

\section{Summary and Discussion}\label{sec:conclusion}

The Stochastic Batch Perceptron is a novel approach for training kernelized
SVMs. The SBP fares well empirically, and, as summarized in Table
\ref{tab:bounds}, our runtime guarantee for the SBP is the best of any existing
guarantee for kernelized SVM training. An interesting open question is whether
this runtime is optimal, i.e. whether any algorithm relying only on black-box
kernel accesses must perform $\Omega\left( ( (\mathcal{L}^*+\gendelta)/\gendelta )^3
\norm{u}^4/\gendelta \right)$ kernel evaluations.


As with other stochastic gradient methods, deciding when to terminate SBP
optimization is an open issue. The most practical approach seems to be to
terminate when a holdout error stabilizes. We should note that even for methods
where the duality gap can be used (e.g.~SMO), this criterion is often too
strict, and the use of cruder criteria may improve training time.

%


\setlength{\bibspacing}{0.7em}
{
\small
\paragraph{Acknowledgements:} S. Shalev-Shwartz is supported by the
Israeli Science Foundation grant number 590-10.

\bibliography{main}
\bibliographystyle{icml2012}
}

\clearpage
\appendix
\inappendixtrue
\section{Additional Experiments}\label{app:experiments}

\begin{figure*}[*t]

\centering
\begin{tabular}{ @{} L @{} S @{} S @{} S @{} }
& \large{Reuters} & \large{Adult} & \large{MNIST} \\
\rotatebox{90}{\scriptsize{Test error}} &
\includegraphics[width=0.30\textwidth]{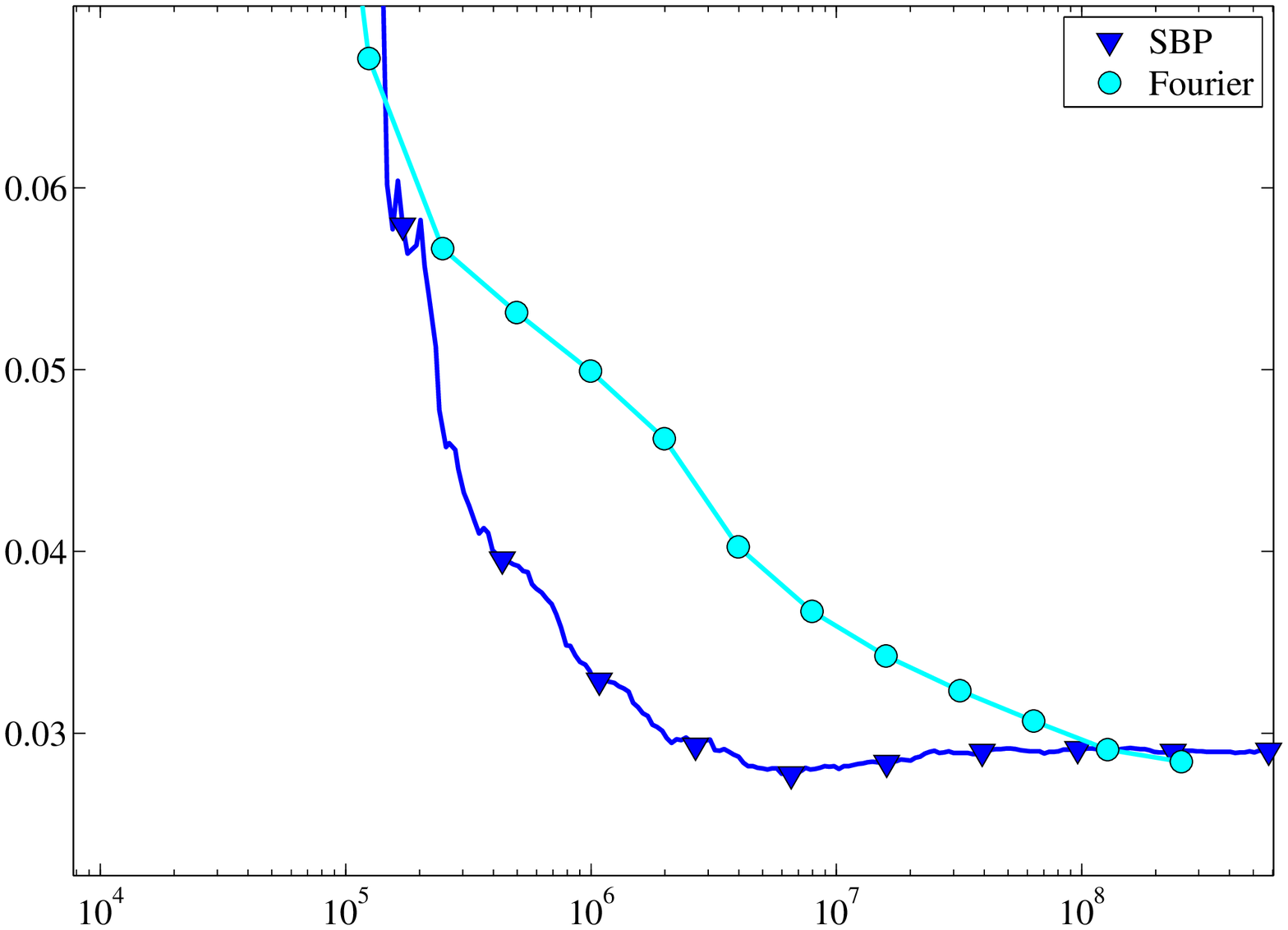} &
\includegraphics[width=0.30\textwidth]{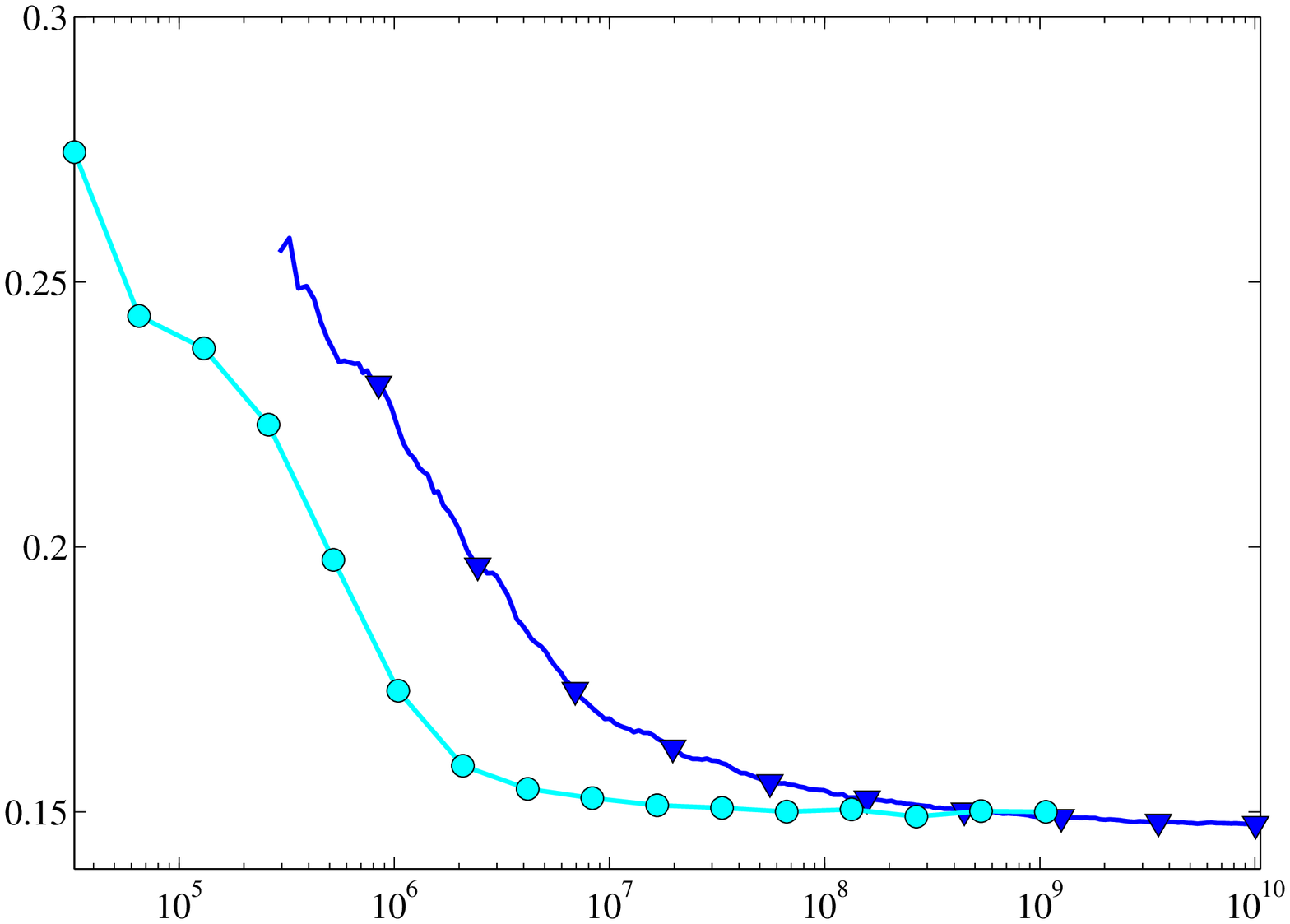} &
\includegraphics[width=0.30\textwidth]{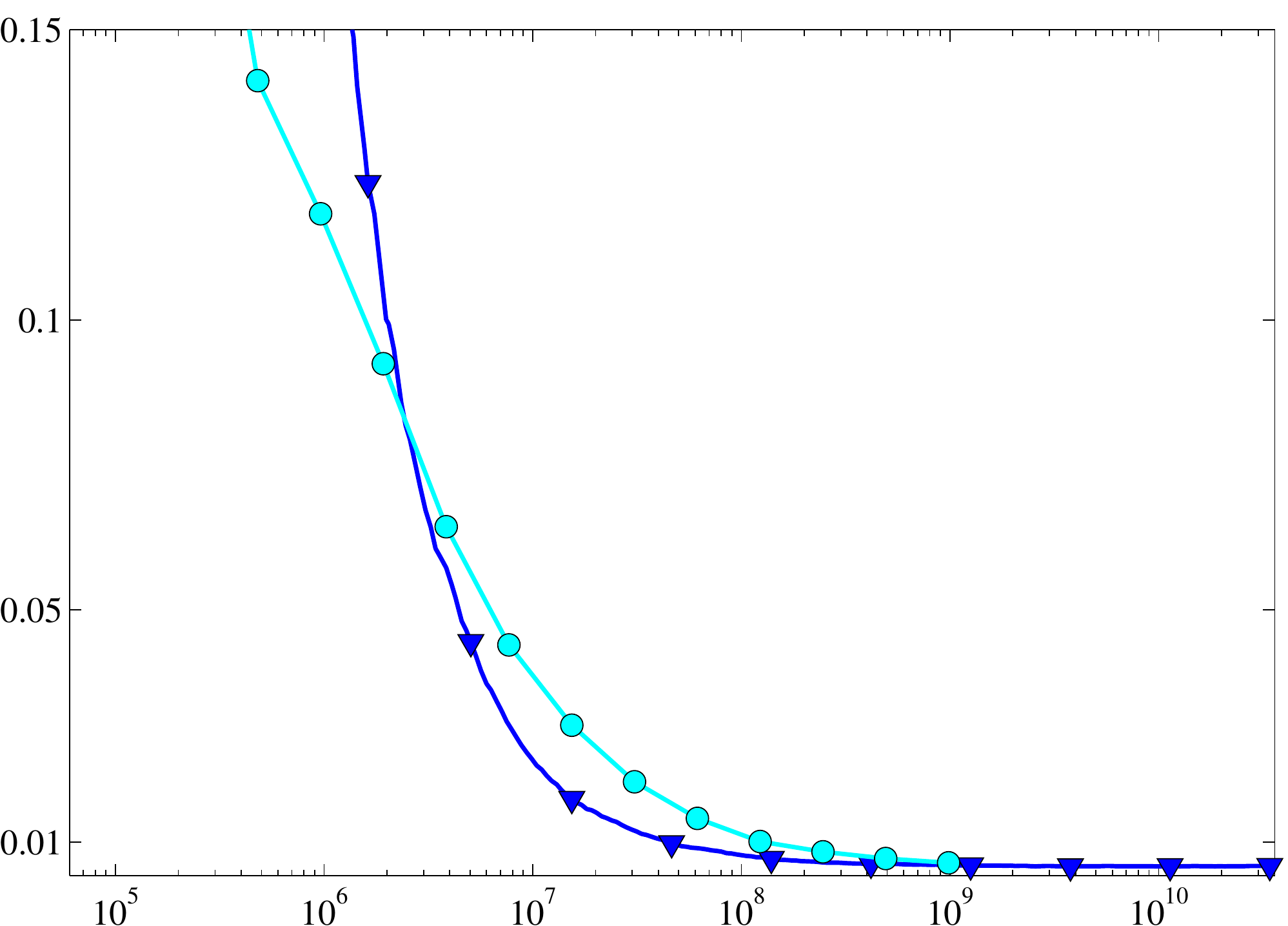} \\
& \scriptsize{Inner Products} & \scriptsize{Inner Products} & \scriptsize{Inner Products} \\
\end{tabular}

\caption{
\small
Classification error on the held-out testing set (linear scale) vs.
computational cost measured in units of $d$-dimensional inner products (where
the training vectors satisfy $x\in\mathbb{R}^{d}$) (log scale), and averaged
over ten runs. For the Fourier features, the computational cost (horizontal
axis) is that of computing $k\in\left\{1,2,4,8,\dots\right\}$ pairs of Fourier
features over the entire training set, while the test error is that of the
\emph{optimal} classifier trained on the resulting linearized SVM objective.
}

\label{fig:fourier-experiments}

\end{figure*}

While our focus in this paper is on optimization of the kernel SVM objective,
and not on the broader problem of large-scale learning, one may wonder how well
the SBP compares to techniques which accelerate the training of kernel SVMs
through \emph{approximation}. One such is the random Fourier projection
algorithm of \citet{RahimiRe07}, which can be used to transform a kernel SVM
problem into an approximately-equivalent linear SVM. The resulting problem may
then be optimized using one of the many existing fast linear SVM solvers, such
as Pegasos, SDCA or SIMBA. Unlike methods (such as the SBP) which rely only on
black-box kernel accesses, Rahimi and Recht's projection technique can only be
applied on a certain class of kernel functions (shift-invariant kernels), of
which the Gaussian kernel is a member.

For $d$-dimensional feature vectors, and using a Gaussian kernel with parameter
$\sigma^2$, Rahimi and Recht's approach is to sample
$v_{1},\dots,v_{k}\in\mathbb{R}^{d}$ independently according to
$v_{i}\sim\mathcal{N}\left(0,I\right)$, and then define the mapping
$\mathcal{P}:\mathbb{R}^{d}\rightarrow\mathbb{R}^{2k}$ as:
\begin{align*}
{\mathcal{P}\left(x\right)}_{2i} = & \frac{1}{\sqrt{k}} \cos\left(
\frac{1}{\sigma} \inner{v_{i}}{x} \right) \\
{\mathcal{P}\left(x\right)}_{2i+1} = & \frac{1}{\sqrt{k}} \sin\left(
\frac{1}{\sigma} \inner{v_{i}}{x} \right)
\end{align*}
Then $\inner{\mathcal{P}\left(x_i\right)}{\mathcal{P}\left(x_j\right)} \approx
K\left(x_i,x_j\right)$, with the quality of this approximation improving with
increasing $k$ (see \citet[Claim 1]{RahimiRe07} for details).

Notice that computing each pair of Fourier features requires computing the
$d$-dimensional inner product $\inner{v}{x}$. For comparison, let us write the
Gaussian kernel in the following form:
\begin{align*}
K\left(x_i,x_j\right) =& \exp\left( -\frac{1}{2\sigma^{2}} \norm{x_i - x_j}^2
\right) \\ =& \exp\left( -\frac{1}{2\sigma^{2}} \left( \norm{x_i}^2 +
\norm{x_j}^2 - 2\inner{x_i}{x_j} \right) \right)
\end{align*}
The norms $\norm{x_i}$ may be cheaply precomputed, so the dominant cost of
performing a single Gaussian kernel evaluation is, likewise, that of the
$d$-dimensional inner product $\inner{x_i}{x_j}$.

This observation suggests that the computational cost of the use of Fourier
features may be directly compared with that of a kernel-evaluation-based SVM
optimizer in terms of $d$-dimensional inner products. Figure
\ref{fig:fourier-experiments} contains such a comparison. In this figure, the
computational cost of a $2k$-dimensional Fourier linearization is taken to be
the cost of computing $\mathcal{P}\left(x_i\right)$ on the entire training set
($kn$ inner products, where $n$ is the number of training examples)---we ignore
the cost of optimizing the resulting linear SVM entirely.  The plotted testing
error is that of the \emph{optimum} of the resulting linear SVM problem, which
approximates the original kernel SVM. We can see that at least on Reuters and
MNIST, the SBP is preferable to (i.e. faster than) approximating the kernel
with random Fourier features.

\section{Implementation Details}\label{app:algorithm}

\begin{algorithm*}[t]

\begin{tabbing}
\textbf{mm}\=mm\=mm\=mm\=mm\=mm\=mm\=mm\=mm\=\kill
\>$\code{optimize}\left( n:\mathbb{N}, x_{1},\dots,x_{n}:\mathbb{R}^{d}, y_{1},\dots,y_{n}:\left\{\pm 1\right\}, T_{0}:\mathbb{N}, T:\mathbb{N}, \nu:\mathbb{R}_{+}, K:\mathbb{R}^{d}\times\mathbb{R}^{d}\rightarrow\mathbb{R}_{+} \right)$\\
\>\textbf{1}\'\>$\eta_{0} := 1 / \sqrt{\max_{i}K\left(x_{i},x_{i}\right)}$;\\
\>\textbf{2}\'\>$\alpha^{(0)} := 0^{n}$; $c^{(0)} := 0^{n}$; $r_{0} := 0$;\\
\>\textbf{3}\'\>$\code{for } t := 1 \code{ to } T$\\
\>\textbf{4}\'\>\>$\eta_{t} := \eta_{0}/\sqrt{t}$;\\
\>\textbf{5}\'\>\>$\gamma := \code{find\_gamma}\left( c^{(t-1)}, n\nu \right)$;\\
\>\textbf{6}\'\>\>$\code{sample } i \sim \code{uniform} \left\{ j : c^{(t-1)}_{j} < \gamma \right\}$;\\
\>\textbf{7}\'\>\>$\alpha^{(t)} := \alpha^{(t-1)} + \eta_{t} e_{i}$;\\
\>\textbf{8}\'\>\>$r_{t}^2 := r_{t-1}^2 + 2 \eta_{t} c^{(t-1)}_{i} + \eta_{t}^{2} K\left(x_{i},x_{i}\right)$;\\
\>\textbf{9}\'\>\>$\code{for } j=1 \code{ to } n$\\
\>\textbf{10}\'\>\>\>$c^{(t)}_{j} := c^{(t-1)}_{j} + \eta_{t} y_{i} y_{j} K\left(x_{i},x_{j}\right)$;\\
\>\textbf{11}\'\>\>$\code{if } \left( r_{t} > 1 \right) \code{ then}$\\
\>\textbf{12}\'\>\>\>$\alpha^{(t)} := \left(1/r_{t}\right) \alpha^{(t)}$; $c^{(t)} := \left(1/r_{t}\right) c^{(t)}$; $r_{t} := 1$;\\
\>\textbf{13}\'\>$\bar{\alpha} := \frac{1}{T} \sum_{t=1}^{T} \alpha^{(t)}$; $\bar{c} := \frac{1}{T} \sum_{t=1}^{T} c^{(t)}$; $\gamma := \code{find\_gamma}\left( \bar{c}, n\nu \right)$;\\
\>\textbf{14}\'\>$\code{return } \bar{\alpha}/\gamma$;
\end{tabbing}

\caption{
\small
Stochastic gradient ascent algorithm for optimizing the kernelized version of
Problem \ref{eq:slack-constrained-objective}, as described in Section
\ref{subsec:slack-constrained-kernel}. Here, $e_{i}$ is the $i$th standard unit
basis vector. The $\code{find\_gamma}$ subroutine finds the ``water level''
$\gamma$ from the vector of responses $c$ and total volume $n\nu$.
}

\label{alg:slack-constrained-sgd}

\end{algorithm*}

\begin{algorithm*}[t]

\begin{tabbing}
\textbf{mm}\=mm\=mm\=mm\=mm\=mm\=mm\=mm\=mm\=\kill
\>$\code{find\_gamma}\left( C:\mathbb{R}^n, n\nu:\mathbb{R} \right)$\\
\>\textbf{1}\'\>$lower := 1$; $upper := n$;\\
\>\textbf{2}\'\>$lower\_max := -\infty$; $lower\_sum := 0$;\\
\>\textbf{3}\'\>$\code{while } lower < upper$\\
\>\textbf{4}\'\>\>$\code{while } lower < upper$;\\
\>\textbf{5}\'\>\>$middle := \code{partition}( C\left[ lower:upper \right] )$;\\
\>\textbf{6}\'\>\>$middle\_max := \max\left( lower\_max, C\left[ lower:\left( middle - 1 \right) \right] \right)$;\\
\>\textbf{7}\'\>\>$middle\_sum := lower\_sum + \sum C\left[ lower:\left( middle - 1 \right) \right]$;\\
\>\textbf{8}\'\>\>$\code{if } middle\_max \cdot \left( middle - 1 \right) - middle\_sum \ge n\nu \code{ then}$\\
\>\textbf{9}\'\>\>\>$upper := middle - 1$;\\
\>\textbf{10}\'\>\>$else$\\
\>\textbf{11}\'\>\>\>$lower := middle$; $lower\_max := middle\_ max$; $lower\_sum := middle\_sum$;\\
\>\textbf{12}\'\>$\code{return } \left( n\nu - lower\_max \cdot \left( lower - 1 \right) + lower\_sum \right) / \left( lower - 1 \right) + lower\_max$;
\end{tabbing}

\caption{
\small
Divide-and-conquer algorithm for finding the ``water level'' $\gamma$ from an
array of responses $C$ and total volume $n\nu$. The $\code{partition}$ function
chooses a pivot value from the array it receives as an argument (the median
would be ideal), places all values less than the pivot at the start of the
array, all values greater at the end, and returns the index of the pivot in the
resulting array.
}

\label{alg:water}

\end{algorithm*}

We begin this appendix by providing complete pseudo-code, which may be found in
Algorithm \ref{alg:slack-constrained-sgd}, for the SBP algorithm which we
outlined in Section \ref{subsec:slack-constrained-sgd}. This implementation
requires that we be able to find a minimax-optimal probability distribution
$p^{*}$ to the objective of Equation \ref{eq:fw-definition}.

As was discussed in Section \ref{subsec:minimax-optimality}, in a problem
without an unregularized bias, such a probability distribution can be derived
from the ``water level'' $\gamma$, which can be found in $O(n)$ time using
Algorithm \ref{alg:water}.
This algorithm works by subdividing the set of responses into those less than,
equal to and greater than a pivot value (if one uses the median, which can be
found in linear time using e.g. the median-of-medians algorithm
\citep{BlumFlPrRiTa73}, then the overall will be linear in $n$). Then, it
calculates the size, minimum and sum of each of these subsets, from which the
total volume of the water required to cover the subsets can be easily
calculated. It then recurses into the subset containing the point at which a
volume of $n\nu$ just suffices to cover the responses, and continues until
$\gamma$ is found.

\begin{figure*}[t]

\centering
\includegraphics[width=0.8\textwidth]{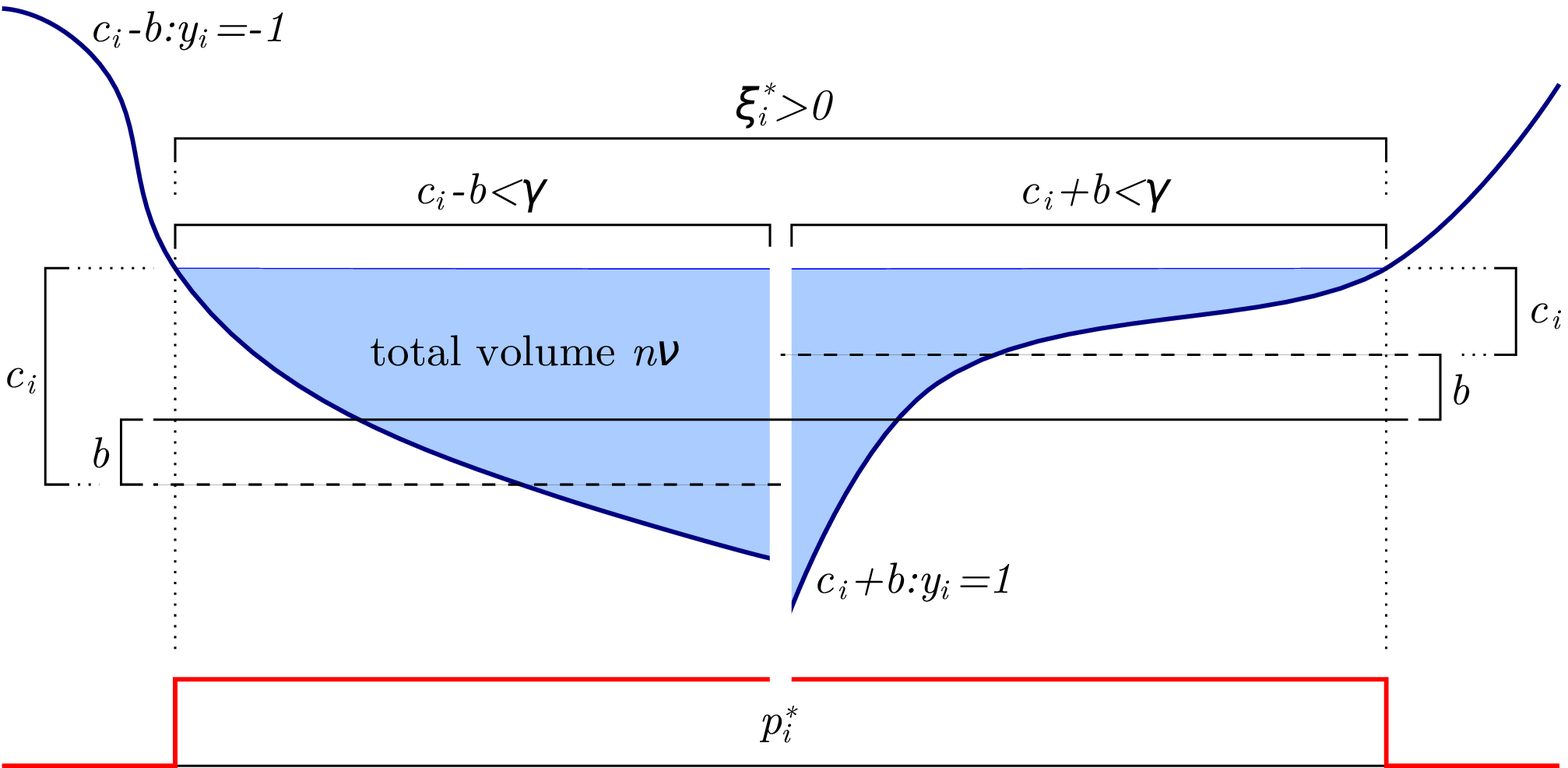}

\caption{
\small
Illustration of how one finds the ``water level'' in a problem with an
unregularized bias. The two curves represent the heights of two basins of
heights $c_{i}-b$ and $c_{i}+b$, corresponding to the negative and positive
examples, respectively, with the bias $b$ determining the relative heights of
the basins. Optimizing over $\xi$ and $p$ corresponds to filling these two
basins with water of total volume $n\nu$ and common water level $\gamma$, while
optimizing $b$ corresponds to ensuring that water covers the same number of
indices in each basin.
}

\label{fig:bias}

\end{figure*}

\begin{algorithm*}[t]

\begin{tabbing}
\textbf{mm}\=mm\=mm\=mm\=mm\=mm\=mm\=mm\=mm\=\kill
\>$\code{find\_gamma\_and\_bias}\left( y:\left\{\pm1\right\}^n, C:\mathbb{R}^n, n\nu:\mathbb{R} \right)$\\
\>\textbf{1}\'\>$C^+ := \{ C[i] : y[i] = +1 \}$; $n^+ := \abs{C^+}$; $lower^+ := 1$; $upper^+ := n^+$; $lower\_max^+ := -\infty$; $lower\_sum^+ := 0$;\\
\>\textbf{2}\'\>$C^- := \{ C[i] : y[i] = -1 \}$; $n^- := \abs{C^-}$; $lower^- := 1$; $upper^- := n^-$; $lower\_max^- := -\infty$; $lower\_sum^- := 0$;\\
\>\textbf{3}\'\>$middle^+ := \code{partition}( C^+\left[ lower^+:upper^+ \right] )$;\\
\>\textbf{4}\'\>$middle^- := \code{partition}( C^-\left[ lower^-:upper^- \right] )$;\\
\>\textbf{5}\'\>$middle\_max^+ := \max\left( C\left[ lower^+:\left( middle^+ - 1 \right) \right] \right)$; $middle\_sum^+ := \sum C\left[ lower^+:\left( middle^+ - 1 \right) \right]$;\\
\>\textbf{6}\'\>$middle\_max^- := \max\left( C\left[ lower^-:\left( middle^- - 1 \right) \right] \right)$; $middle\_sum^- := \sum C\left[ lower^-:\left( middle^- - 1 \right) \right]$;\\
\>\textbf{7}\'\>$\code{while } \left( lower^+ < upper^+ \right) \code{ or } \left( lower^- < upper^- \right)$\\
\>\textbf{8}\'\>\>$direction^+ := 0$; $direction^- := 0$;\\
\>\textbf{9}\'\>\>$\code{if } middle^+ < lower^- \code{ then } direction^+ = 1$;\\
\>\textbf{10}\'\>\>$\code{else if } middle^+ > upper^- \code{ then } direction^+ = -1$;\\
\>\textbf{11}\'\>\>$\code{if } middle^- < lower^+ \code{ then } direction^- = 1$;\\
\>\textbf{12}\'\>\>$\code{else if } middle^- > upper^+ \code{ then } direction^- = -1$;\\
\>\textbf{13}\'\>\>$\code{if } direction^+ = direction^- = 0 \code{ then}$\\
\>\textbf{14}\'\>\>\>$volume^+ := middle\_max^+ \cdot \left( middle^+ - 1 \right) - middle\_sum^+$;\\
\>\textbf{15}\'\>\>\>$volume^- := middle\_max^- \cdot \left( middle^- - 1 \right) - middle\_sum^-$;\\
\>\textbf{16}\'\>\>\>$\code{if } volume^+ + volume^- \ge n\nu \code{ then}$\\
\>\textbf{17}\'\>\>\>\>$\code{if } middle^+ > middle^- \code{ then } direction^+ = -1$;\\
\>\textbf{18}\'\>\>\>\>$\code{else if } middle^- > middle^+ \code{ then } direction^- = -1$;\\
\>\textbf{19}\'\>\>\>\>$\code{else if } upper^+ - lower^+ > upper^- - lower^- \code{ then } direction^+ = -1$;\\
\>\textbf{20}\'\>\>\>\>$\code{else } direction^- = -1$;\\
\>\textbf{21}\'\>\>\>$\code{else}$\\
\>\textbf{22}\'\>\>\>\>$\code{if } middle^+ < middle^- \code{ then } direction^+ = 1$;\\
\>\textbf{23}\'\>\>\>\>$\code{else if } middle^- < middle^+ \code{ then } direction^- = 1$;\\
\>\textbf{24}\'\>\>\>\>$\code{else if } upper^+ - lower^+ > upper^- - lower^- \code{ then } direction^+ = 1$;\\
\>\textbf{25}\'\>\>\>\>$\code{else } direction^- = 1$;\\
\>\textbf{26}\'\>\>$\code{if } direction^+ \ne 0 \code{ then}$\\
\>\textbf{27}\'\>\>\>$\code{if } direction^+ > 0 \code{ then } upper^+ := middle^+ - 1$;\\
\>\textbf{28}\'\>\>\>$\code{else } lower^+ := middle^+$; $lower\_max^+ := middle\_max^+$; $lower\_sum^+ := middle\_sum^+$;\\
\>\textbf{29}\'\>\>\>$middle^+ := \code{partition}( C^+\left[ lower^+:upper^+ \right] )$;\\
\>\textbf{30}\'\>\>\>$middle\_max^+ := \max\left( lower\_max^+, C\left[ lower^+:\left( middle^+ - 1 \right) \right] \right)$;\\
\>\textbf{31}\'\>\>\>$middle\_sum^+ := lower\_sum^+ + \sum C\left[ lower^+:\left( middle^+ - 1 \right) \right]$;\\
\>\textbf{32}\'\>\>$\code{if } direction^- \ne 0 \code{ then}$\\
\>\textbf{33}\'\>\>\>$\code{if } direction^- > 0 \code{ then } upper^- := middle^- - 1$;\\
\>\textbf{34}\'\>\>\>$\code{else } lower^- := middle^-$; $lower\_max^- := middle\_max^-$; $lower\_sum^- := middle\_sum^-$;\\
\>\textbf{35}\'\>\>\>$middle^- := \code{partition}( C^-\left[ lower^-:upper^- \right] )$;\\
\>\textbf{36}\'\>\>\>$middle\_max^- := \max\left( lower\_max^-, C\left[ lower^-:\left( middle^- - 1 \right) \right] \right)$;\\
\>\textbf{37}\'\>\>\>$middle\_sum^- := lower\_sum^- + \sum C\left[ lower^-:\left( middle^- - 1 \right) \right]$;\\
\>\textbf{38}\'\>// at this point $lower^+ = lower^- = upper^+ = upper^-$\\
\>\textbf{39}\'\>$\Delta\gamma := \left( n\nu + lower\_sum^+ + lower\_sum^- \right) / \left( lower^+ - 1 \right) - lower\_max^+ - lower\_max^-$;\\
\>\textbf{40}\'\>$\code{if } lower^+ < n^+ \code{ then } \Delta\gamma^+ := \min\left( \Delta\gamma, C^+[ lower^+ ] - lower\_max^+ \right) \code{ else } \Delta\gamma^+ := \Delta\gamma$;\\
\>\textbf{41}\'\>$\code{if } lower^- < n^- \code{ then } \Delta\gamma^- := \min\left( \Delta\gamma, C^-[ lower^- ] - lower\_max^- \right) \code{ else } \Delta\gamma^- := \Delta\gamma$;\\
\>\textbf{42}\'\>$\gamma^+ := lower\_max^+ + 0.5 \cdot \left( \Delta\gamma + \Delta\gamma^+ -\Delta\gamma^- \right)$; $\gamma^- := lower\_max^- + 0.5 \cdot \left( \Delta\gamma - \Delta\gamma^+ +\Delta\gamma^- \right)$;\\
\>\textbf{43}\'\>$\gamma := 0.5 \cdot \left( \gamma^+ + \gamma^- \right)$; $b := 0.5 \cdot \left( \gamma^- - \gamma^+ \right)$;\\
\>\textbf{44}\'\>$\code{return } \left( \gamma, b \right)$;\\
\end{tabbing}

\caption{
\small
Divide-and-conquer algorithm for finding the ``water level'' $\gamma$ and bias
$b$ from an array of labels $y$, array of responses $C$ and total volume
$n\nu$, for a problem with an unregularized bias. The $\code{partition}$
function is as in Algorithm \ref{alg:water}.
}

\label{alg:water-bias}

\end{algorithm*}

In Section \ref{subsec:unregularized-bias}, we mentioned that a similar result
holds for the objective of Equation \ref{eq:fw-definition-bias}, which adds an
unregularized bias.

As before, finding the water level $\gamma$ reduces to finding minimax-optimal
values of $p^{*}$, $\xi^{*}$ and $b^{*}$. The characterization of such
solutions is similar to that in the case without an unregularized bias.  In
particular, for a fixed value of $b$, we may still think about ``pouring water
into a basin'', except that the height of the basin is now $c_{i}+y_{i}b$,
rather than $c_{i}$.

When $b$ is not fixed it is easier to think of \emph{two} basins, one
containing the positive examples, and the other the negative examples. These
basins will be filled with water of a total volume of $n\nu$, to a common water
level $\gamma$. The relative heights of the two basins are determined by $b$:
increasing $b$ will raise the basin containing the positive examples, while
lowering that containing the negative examples by the same
amount. This is illustrated in Figure \ref{fig:bias}.

It remains only to determine what characterizes a minimax-optimal value of $b$.
Let $k^{+}$ and $k^{-}$ be the number of elements covered by water in the
positive and negative basins, respectively, for some $b$. If $k^{+}>k^{-}$,
then raising the positive basin and lowering the negative basin by the same
amount (i.e.  increasing $b$) will raise the overall water level, showing that
$b$ is not optimal. Hence, for an optimal $b$, water must cover an equal number
of indices in each basin. Similar reasoning shows that an optimal $p^{*}$ must
place equal probability mass on each of the two classes.

Once more, the resulting problem is amenable to a divide-and-conquer approach.
The water level $\gamma$ and bias $b$ will be found in $O(n)$ time by Algorithm
\ref{alg:water-bias}, provided that the $\code{partition}$ function chooses the
median as the pivot.

\section{Proofs of Lemmas \ref{lem:slack-constrained-supergradient} and \ref{lem:zinkevich-batch}}\label{app:proofs}

\ifinappendix
	\begin{replem}{lem:slack-constrained-supergradient}
\else
	\begin{lem}[Proof in Appendix \ref{app:proofs}]\label{lem:slack-constrained-supergradient}
\fi

For any $w$, let $p^*,\xi^*$ be minimax optimal for Equation
\ref{eq:fw-definition}. Then $\sum_{i=1}^{n}
p^{*}_{i}y_{i}\Phi\left(x_{i}\right)$ is a supergradient of $f(w)$ at $w$.

\ifinappendix
	\end{replem}
\else
	\end{lem}
\fi

\begin{proof}

By the definition of $f$, for any $v\in\mathbb{R}^{d}$:
\begin{align*}
\MoveEqLeft f\left(w+v\right) = \\
& \max_{\xi\succeq 0, \mathbf{1}^{T} \xi \le n\nu}\min_{p\in\Delta^{n}}
\sum_{i=1}^{n} p_{i} \left( y_{i}\inner{w+v}{\Phi\left(x_{i}\right)} + \xi_{i}
\right)
\end{align*}
Substituting the particular value $p^{*}$ for $p$ can only increase the RHS,
so:
\begin{align*}
f\left(w+v\right) \le & \max_{\xi\succeq 0, \mathbf{1}^{T} \xi \le n\nu}
\sum_{i=1}^{n} p^{*}_{i} \left( y_{i}\inner{w+v}{\Phi\left(x_{i}\right)} +
\xi_{i} \right) \\
\le & \max_{\xi\succeq 0, \mathbf{1}^{T} \xi \le n\nu} \sum_{i=1}^{n} p^{*}_{i}
\left( y_{i}\inner{w}{\Phi\left(x_{i}\right)} + \xi_{i} \right) \\
& + \sum_{i=1}^{n} p^{*}_{i}y_{i}\inner{v}{\Phi\left(x_{i}\right)}
\end{align*}
Because $p^{*}$ is minimax-optimal at $w$:
\begin{align*}
f\left(w+v\right) \le & f\left(w\right) + \sum_{i=1}^{n}
p^{*}_{i}y_{i}\inner{v}{\Phi\left(x_{i}\right)} \\
\le & f\left(w\right) +
\inner{v}{\sum_{i=1}^{n}p^{*}_{i}y_{i}\Phi\left(x_{i}\right)}
\end{align*}
So $\sum_{i=1}^{n} p^{*}_{i}y_{i}\Phi\left(x_{i}\right)$ is a supergradient of
$f$.
\end{proof}

\ifinappendix
	\begin{replem}{lem:zinkevich-batch}
\else
	\begin{lem}[Proof in Appendix \ref{app:proofs}]\label{lem:zinkevich-batch}
\fi

For any $T,\delta>0$, after $T$ iterations of the Stochastic Batch
Perceptron, with probability at least $1-\delta$, the average iterate
$\bar{w} = \frac{1}{T}\sum_{t=1}^{T} w^{(t)}$ (corresponding to
$\bar{\alpha} = \frac{1}{T} \sum_{t=1}^{T} \alpha^{(t)}$), satisfies:
%
%
$
f\left(\bar{w}\right) \geq 
\sup_{\norm{w}\leq 1} f\left(w \right) -
O\left(\sqrt{\frac{1}{T}\log\frac{1}{\delta}}\right).
$
%

\ifinappendix
	\end{replem}
\else
	\end{lem}
\fi

\begin{proof}

Define $h=-\frac{1}{r}f$, where $f$ is as in Equation \ref{eq:fw-definition}.
Then the stated update rules constitute an instance of Zinkevich's algorithm,
in which steps are taken in the direction of stochastic subgradients $g^{(t)}$
of $h$ at $w^{(t)}=\sum_{i=1}^{n}\alpha_{i}y_{i}\Phi\left(x_{i}\right)$.

The claimed result follows directly from \citet[Theorem 1]{Zinkevich03}
combined with an online-to-batch conversion analysis in the style of
\citet[Lemma 1]{CesaCoGe01}.
\end{proof}

\section{Data-Laden Analyses}\label{app:other-algorithms}

We'll begin by presenting a bound on the sample size $n$ required to guarantee
good generalization performance (in terms of the 0/1 loss) for a classifier
which is $\epsilon$-suboptimal in terms of the empirical hinge loss. The
following result, which follows from \citet[Theorem 1]{SrebroSrTe10}, is a
vital building block of the bounds derived in the remainder of this appendix:

\medskip
\begin{lem}
\label{lem:generalization-from-expected-loss}

Consider the expected 0/1 and hinge losses:
\begin{align*}
\mathcal{L}_{0/1}\left(w\right) &=
\expectation[x,y]{\mathbf{1}_{y\inner{w}{x}\le0}}\\
\mathcal{L}\left(w\right) &=
\expectation[x,y]{\max\left(0,1-y\inner{w}{x}\right)}
\end{align*}
Let $u$ be an arbitrary linear classifier, and suppose that we sample a
training set of size $n$, with $n$ given by the following equation, for
parameters $B \ge \norm{u}$, $\epsilon>0$ and
$\delta\in\left(0,1\right)$:
\begin{equation}
\label{eq:generalization-from-expected-loss-bound} n = \tilde{O}\left( \left(
\frac{\mathcal{L}\left(u\right) + \epsilon}{\epsilon} \right) \frac{ \left( B +
\sqrt{\log\frac{1}{\delta}} \right)^{2} + r B \log\frac{1}{\delta} }{\epsilon}
\right)
\end{equation}
where $r \ge \norm{x}$ is an upper bound on the radius of the data.
Then, with probability $1-\delta$ over the i.i.d.~training sample
$x_{i},y_{i}:i\in\left\{ 1,\dots,n\right\}$, uniformly for all linear
classifiers $w$ satisfying:
\begin{align*}
\norm{w} & \le B \\
\hat{\mathcal{L}}\left(w\right) - \hat{\mathcal{L}}\left(u\right) &
\le \epsilon
\end{align*}
where $\hat{\mathcal{L}}$ is the empirical hinge loss:
\begin{equation*}
\hat{\mathcal{L}}\left(w\right) =
\frac{1}{n}\sum_{i=1}^{n}\max\left(0,1-y_{i}\inner{w}{x_{i}}\right)\\
\end{equation*}
we have that:
\begin{align*}
\hat{\mathcal{L}}\left(u\right) & \le \mathcal{L}\left(u\right) + \epsilon \\
\mathcal{L}_{0/1}\left(w\right) & \le \hat{\mathcal{L}}\left(u\right) + \epsilon
\end{align*}
and in particular that:
\begin{equation*}
\mathcal{L}_{0/1}\left(w\right)\le \mathcal{L}\left(u\right) + 2\epsilon
\end{equation*}

\end{lem}

In the remainder of this appendix, we will apply the above result to derive 
generalization bounds on the performance of the various algorithms under
consideration, in the data-laden setting.

\subsection{Stochastic Batch Perceptron}\label{subapp:slack-constrained}

We will here present a more careful derivation of the main result of Section
\ref{subsec:slack-constrained-runtime}, bounding the generalization performance
of the SBP.

\begin{thm}
\label{thm:slack-constrained-runtime}

Let $u$ be an arbitrary linear classifier in the RKHS, let $\epsilon>0$ be
given, and suppose that $K\left(x,x\right)\le r^{2}$ with probability $1$.
There exist values of the training size $n$, iteration count $T$ and parameter
$\nu$ such that Algorithm \ref{alg:slack-constrained-sgd} finds a solution $w =
\sum_{i=1}^{n} \alpha_{i}y_{i}\Phi\left(x_{i}\right)$ satisfying:
\begin{equation*}
\mathcal{L}_{0/1}\left(w\right) \le \mathcal{L}\left(u\right) + \epsilon
\end{equation*}
where $\mathcal{L}_{0/1}$ and $\mathcal{L}$ are the expected 0/1 and hinge
losses, respectively, after performing the following number of kernel
evaluations:
\begin{equation*}
\mbox{\#K} = \tilde{O}\left( \left( \frac{\mathcal{L}\left(u\right) +
\epsilon}{\epsilon} \right)^{3} \frac{ r^{3} \norm{u}^{4} }{\epsilon}
\log^{2}\frac{1}{\delta} \right)
\end{equation*}
with the size of the support set of $w$ (the number nonzero elements in
$\alpha$) satisfying:
\begin{equation*}
\mbox{\#S} = O\left( \left( \frac{\mathcal{L}\left(u\right) +
\epsilon}{\epsilon} \right)^{2} r^{2} \norm{u}^{2} \log\frac{1}{\delta} \right)
\end{equation*}
the above statements holding with probability $1-\delta$.

\end{thm}

\begin{proof}

For a training set of size $n$, where:
\begin{equation*}
n = \tilde{O}\left( \left( \frac{\mathcal{L}\left(u\right) +
\epsilon}{\epsilon} \right) \frac{ r B^{2} }{\epsilon} \log\frac{1}{\delta}
\right)
\end{equation*}
taking $B=2\norm{u}$ in Lemma \ref{lem:generalization-from-expected-loss} gives
that $\hat{\mathcal{L}}\left(u\right) \le \mathcal{L}\left(u\right) + \epsilon$
and $\mathcal{L}_{0/1}\left(w\right) \le \mathcal{L}\left(u\right) + 2\epsilon$
with probability $1-\delta$ over the training sample, uniformly for all linear
classifiers $w$ such that $\norm{w} \le B$ and $\hat{\mathcal{L}}\left(w\right)
- \hat{\mathcal{L}}\left( u \right) \le \epsilon$, where $\hat{\mathcal{L}}$ is
the empirical hinge loss. We will now show that these inequalities are
satisfied by the result of Algorithm \ref{alg:slack-constrained-sgd}. Define:
\begin{equation*}
\hat{w}^{*} = \argmin{w:\norm{w}\le\norm{u}} \hat{\mathcal{L}}\left(w\right)
\end{equation*}
Because $\hat{w}^{*}$ is a Pareto optimal solution of the bi-criterion
objective of Problem \ref{eq:bi-criterion-objective}, if we choose the
parameter $\nu$ to the slack-constrained objective (Problem
\ref{eq:slack-constrained-objective}) such that $\norm{\hat{w}^{*}}\nu =
\hat{\mathcal{L}}\left(\hat{w}^{*}\right)$, then the optimum of the
slack-constrained objective will be equivalent to $\hat{w}^{*}$ (Lemma
\ref{lem:slack-constrained-suboptimality}). As was discussed in Section
\ref{subsec:slack-constrained-runtime}, We will use Lemma
\ref{lem:zinkevich-batch} to find the number of iterations $T$ required to
satisfy Equation \ref{eq:reqbarepsilon} (with $u=\hat{w}^{*}$).  This yields
that, if we perform $T$ iterations of Algorithm
\ref{alg:slack-constrained-sgd}, where $T$ satisfies the following:
\begin{equation}
\label{eq:slack-constrained-time} T \ge O\left( \left(\frac{\hat{ \mathcal{L}
}\left(\hat{w}^{*}\right) + \epsilon }{\epsilon}\right)^2 r^2
\norm{\hat{w}^{*}}^2 \log\frac{1}{\delta} \right)
\end{equation}
then the resulting solution $w=\bar{w}/\gamma$ will satisfy:
\begin{align*}
\norm{w} & \le 2\norm{\hat{w}^{*}} \\
\notag \hat{\mathcal{L}}\left(w\right) -
\hat{\mathcal{L}}\left(\hat{w}^{*}\right) & \le \epsilon
\end{align*}
with probability $1-\delta$. That is:
\begin{align*}
\norm{w} & \le 2 \norm{\hat{w}^{*}} \\
& \le B
\end{align*}
and:
\begin{align*}
\hat{\mathcal{L}}\left(w\right) & \le \hat{\mathcal{L}}\left( \hat{w}^{*}
\right) + \epsilon \\
& \le \hat{\mathcal{L}}\left( u \right) + \epsilon
\end{align*}
These are precisely the bounds on $\norm{w}$ and
$\hat{\mathcal{L}}\left(w\right)$ which we determined (at the start of the
proof) to be necessary to permit us to apply Lemma
\ref{lem:generalization-from-expected-loss}. Each of the $T$ iterations
requires $n$ kernel evaluations, so the product of the bounds on $T$ and $n$
bounds the number of kernel evaluations (we may express Equation
\ref{eq:slack-constrained-time} in terms of $\mathcal{L}\left(u\right)$ and
$\norm{u}$ instead of $\hat{\mathcal{L}}\left(\hat{w}^*\right)$ and
$\norm{\hat{w}^*}$, since $\hat{\mathcal{L}}\left(\hat{w}^*\right) \le
\hat{\mathcal{L}}\left(u\right) \le \mathcal{L}\left(u\right)+\epsilon$ and
$\norm{\hat{w}^*} \le \norm{u}$).

Because each iteration will add at most one new element to the support set, the
size of the support set is bounded by the number of iterations, $T$.

This discussion has proved that we can achieve suboptimality $2\epsilon$ with
probability $1-2\delta$ with the given $\mbox{\#K}$ and $\mbox{\#S}$. Because
scaling $\epsilon$ and $\delta$ by $1/2$ only changes the resulting bounds by
constant factors, these results apply equally well for suboptimality $\epsilon$
with probability $1-\delta$.
\end{proof}

\subsection{Pegasos / SGD on $\hat{\mathcal{L}}$}\label{subapp:pegasos}

If $w$ is the result of a call to the Pegasos algorithm \citep{ShalevSiSrCo10}
without a projection step, then the analysis of \citet[Corollary 7]{KakadeTe09}
permits us to bound the suboptimality relative to an arbitrary reference
classifier $u$, with probability $1-\delta$, as:
\begin{align}
\label{eq:regularized-bound} \MoveEqLeft \left( \frac{\lambda}{2}\norm{w}^{2} +
\hat{\mathcal{L}}\left(w\right) \right) - \left( \frac{\lambda}{2}\norm{u}^{2}
+ \hat{\mathcal{L}}\left(u\right) \right) \le \\
\notag & \frac{84 r^{2} \log T}{\lambda T} \log\frac{1}{\delta}
\end{align}

Equation \ref{eq:regularized-bound} implies that, if one performs the following
number of iterations, then the resulting solution will be
$\epsilon/2$-suboptimal in the \emph{regularized objective}, with probability
$1-\delta$:
\begin{equation*}
T = \tilde{O}\left( \frac{1}{\epsilon} \cdot \frac{r^2}{\lambda}
\log\frac{1}{\delta} \right)
\end{equation*}
Here, $\epsilon$ bounds the suboptimality not of the empirical hinge loss, but
rather of the regularized objective (hinge loss + regularization). Although the
dependence on $1/\epsilon$ is linear, accounting for the $\lambda$ dependence
results in a bound which is not nearly good as the above appears. To see this,
we'll follow \citet{ShalevSr08} by decomposing the suboptimality in the
empirical hinge loss as:
\begin{align*}
\hat{\mathcal{L}}\left( w \right) - \hat{\mathcal{L}}\left( u \right) & =
\frac{\epsilon}{2} - \frac{\lambda}{2}\norm{w}^2 + \frac{\lambda}{2}\norm{u}^2
\\
& \le \frac{\epsilon}{2} + \frac{\lambda}{2}\norm{u}^2
\end{align*}
In order to have both terms bounded by $\epsilon/2$, we choose $\lambda =
\epsilon / \norm{u}^{2}$, which reduces the RHS of the above to $\epsilon$.
Continuing to use this choice of $\lambda$, we next decompose the squared norm
of $w$ as:
\begin{align*}
\frac{\lambda}{2}\norm{w}^2 & = \frac{\epsilon}{2} - \hat{\mathcal{L}}\left( w
\right) + \hat{\mathcal{L}}\left( u \right) + \frac{\lambda}{2}\norm{u}^2 \\
& \le \frac{\epsilon}{2} + \hat{\mathcal{L}}\left( u \right) +
\frac{\lambda}{2}\norm{u}^2 \\
\norm{w}^2 & \le 2 \left( \frac{ \hat{\mathcal{L}}\left( u \right) + \epsilon
}{ \epsilon} \right) \norm{u}^2
\end{align*}
Hence, we will have that:
\begin{align}
\label{eq:regularized-target} \norm{w}^{2} & \le 2 \left( \frac{
\hat{\mathcal{L}}\left(u\right) + \epsilon }{ \epsilon} \right)
\norm{u}^{2} \\
\notag \hat{\mathcal{L}}\left( w \right) - \hat{\mathcal{L}}\left(u\right) &
\le \epsilon
\end{align}
with probability $1-\delta$, after performing the following number of
iterations:
\begin{equation}
\label{eq:regularized-time} T = \tilde{O}\left( \frac{r^2
\norm{u}^2}{\epsilon^{2}} \log\frac{1}{\delta} \right)
\end{equation}

There are two ways in which we will use this bound on $T$ to find bound on the
number of kernel evaluations required to achieve some desired regularization
error. The easiest is to note that the bound of Equation
\ref{eq:regularized-time} exceeds that of Lemma
\ref{lem:generalization-from-expected-loss}, so that if we take $T=n$, then
with high probability, we'll achieve generalization error $2 \epsilon$ after
$Tn=T^{2}$ kernel evaluations:
\begin{equation}
\label{eq:regularized-online-runtime} \mbox{\#K} = \tilde{O}\left(
\frac{r^4 \norm{u}^4}{\epsilon^{4}} \log^{2}\frac{1}{\delta} \right)
\end{equation}
Because we take the number of iterations to be precisely the same as the number
of training examples, this is essentially the online stochastic setting.

Alternatively, we may combine our bound on $T$ with Lemma
\ref{lem:generalization-from-expected-loss}. This yields the following bound on
the generalization error of Pegasos in the data-laden batch setting.

\medskip
\begin{thm}
\label{thm:regularized-runtime}

Let $u$ be an arbitrary linear classifier in the RKHS, let $\epsilon>0$ be
given, and suppose that $K\left(x,x\right)\le r^{2}$ with probability $1$.
There exist values of the training size $n$, iteration count $T$ and parameter
$\nu$ such that kernelized Pegasos finds a solution $w = \sum_{i=1}^{n}
\alpha_{i}y_{i}\Phi\left(x_{i}\right)$ satisfying:
\begin{equation*}
\mathcal{L}_{0/1}\left(w\right) \le \mathcal{L}\left(u\right) + \epsilon
\end{equation*}
where $\mathcal{L}_{0/1}$ and $\mathcal{L}$ are the expected 0/1 and hinge
losses, respectively, after performing the following number of kernel
evaluations:
\begin{equation*}
\mbox{\#K} = \tilde{O}\left( \left( \frac{\mathcal{L}\left(u\right) +
\epsilon}{\epsilon} \right)^{2} \frac{ r^{3} \norm{u}^{4} }{\epsilon^{3}}
\log^{2}\frac{1}{\delta} \right)
\end{equation*}
with the size of the support set of $w$ (the number nonzero elements in
$\alpha$) satisfying:
\begin{equation*}
\mbox{\#S} = \tilde{O}\left( \frac{r^{2} \norm{u}^{2}}{\epsilon^2}
\log\frac{1}{\delta} \right)
\end{equation*}
the above statements holding with probability $1-\delta$.

\end{thm}

\begin{proof}

Same proof technique as in Theorem \ref{thm:slack-constrained-runtime}.
\end{proof}

Because of the extra term in the bound on $\norm{w}$ in Equation
\ref{eq:regularized-target}, theorem \ref{thm:regularized-runtime} gives a
bound which is worse by a factor of
$\left(\mathcal{L}\left(u\right)+\epsilon\right) / \epsilon$ than what we might
have hoped to recover.  When $\epsilon \ll \mathcal{L}\left(u\right)$, this
extra factor results in the bound going as $1 / \epsilon^{5}$ rather than $1 /
\epsilon^{4}$. We need to use Equation \ref{eq:regularized-online-runtime} to
get a $1 / \epsilon^{4}$ bound in this case.

Although this bound on the generalization performance of Pegasos is not quite
what we expected, for the related algorithm which performs SGD on the following
objective:
\begin{align*}
\min_{w\in\mathbb{R}^{d}} &
\frac{1}{n}\sum_{i=1}^{n}\ell\left(y_i\inner{w}{\Phi\left(x_{i}\right)}\right)
\\
\mbox{subject to: } & \norm{w}^2 \le B^2 \\
\end{align*}
the same proof technique yields the desired bound (i.e. without the extra
$\left(\mathcal{L}\left(u\right)+\epsilon\right)/\epsilon$ factor). This is the
origin of the ``SGD on $\hat{\mathcal{L}}$'' row in Table \ref{tab:bounds}.

\subsection{Perceptron}\label{subapp:perceptron}

Analysis of the venerable online Perceptron algorithm is typically presented as
a bound on the number of mistakes made by the algorithm in terms of the hinge
loss of the best classifier---this is precisely the form which we consider in
this document, despite the fact that the online Perceptron does not optimize
any scalarization of the bi-criterion SVM objective of Problem
\ref{eq:bi-criterion-objective}.  Interestingly, the performance of the
Perceptron matches that of the SBP, as is shown in the following theorem:

\medskip
\begin{thm}
\label{thm:perceptron-runtime}

Let $u$ be an arbitrary linear classifier in the RKHS, let $\epsilon>0$ be
given, and suppose that $K\left(x,x\right)\le r^{2}$ with probability $1$.
There exists a value of the training size $n$ such that when the Perceptron
algorithm is run for a single ``pass'' over the dataset, the result is a
solution $w = \sum_{i=1}^{n} \alpha_{i}y_{i}\Phi\left(x_{i}\right)$ satisfying:
\begin{equation*}
\mathcal{L}_{0/1}\left(w\right) \le \mathcal{L}\left(u\right) + \epsilon
\end{equation*}
where $\mathcal{L}_{0/1}$ and $\mathcal{L}$ are the expected 0/1 and hinge
losses, respectively, after performing the following number of kernel
evaluations:
\begin{equation*}
\mbox{\#K} = \tilde{O}\left( \left( \frac{\mathcal{L}\left(u\right) +
\epsilon}{\epsilon} \right)^{3} \frac{ r^{4} \norm{u}^{4} }{\epsilon}
\frac{1}{\delta} \right)
\end{equation*}
with the size of the support set of $w$ (the number nonzero elements in
$\alpha$) satisfying:
\begin{equation*}
\mbox{\#S} = O\left( \left( \frac{\mathcal{L}\left(u\right) +
\epsilon}{\epsilon} \right)^{2} r^{2} \norm{u}^{2} \frac{1}{\delta} \right)
\end{equation*}
the above statements holding with probability $1-\delta$.

\end{thm}

\begin{proof}

If we run the online Perceptron algorithm for a single pass over the dataset,
then Corollary 5 of \cite{Shalev07} gives the following mistake bound, for
$\mathcal{M}$ being the set of iterations on which a mistake is made:
\begin{align}
\label{eq:perceptron-starting-bound} \MoveEqLeft \abs{\mathcal{M}} \le
\sum_{i\in\mathcal{M}}\ell\left(y_{i}\inner{u}{\Phi\left(x_{i}\right)}\right) \\
\nonumber & + r \norm{u}
\sqrt{\sum_{i\in\mathcal{M}}\ell\left(y_{i}\inner{u}{\Phi\left(x_{i}\right)}\right)} + r^{2}
\norm{u}^{2} \\
\MoveEqLeft \nonumber \sum_{i=1}^{n}\ell_{0/1}\left(y_{i}\inner{w_{i}}{\Phi\left(x_{i}\right)}\right)
\le \sum_{i=1}^{n}\ell\left(y_{i}\inner{u}{\Phi\left(x_{i}\right)}\right) + \\
\nonumber & + r \norm{u} \sqrt{\sum_{i=1}^{n}\ell\left(y_{i}\inner{u}{\Phi\left(x_{i}\right)}\right)} +
r^{2} \norm{u}^{2}
\end{align}
Here, $\ell$ is the hinge loss and $\ell_{0/1}$ is the 0/1 loss. Dividing
through by $n$:
\begin{align*}
\MoveEqLeft \frac{1}{n}\sum_{i=1}^{n}\ell_{0/1}\left(y_{i}\inner{w_{i}}{\Phi\left(x_{i}\right)}\right)
\le \frac{1}{n}\sum_{i=1}^{n}\ell\left(y_{i}\inner{u}{\Phi\left(x_{i}\right)}\right) \\
\notag & + \frac{r
\norm{u}}{\sqrt{n}}\sqrt{\frac{1}{n}\sum_{i=1}^{n}\ell\left(y_{i}\inner{u}{\Phi\left(x_{i}\right)}\right)}
+ \frac{r^{2} \norm{u}^{2}}{n}
\end{align*}
If we suppose that the $x_{i},y_{i}$s are i.i.d., and that
$w\sim\mbox{Unif}\left(w_{1},\dots,w_{n}\right)$ (this is a ``sampling''
online-to-batch conversion), then:
\begin{equation*}
\expectation{\mathcal{L}_{0/1}\left(w\right)} \le \mathcal{L}\left(u\right) + \frac{r
\norm{u}}{\sqrt{n}}\sqrt{\mathcal{L}\left(u\right)} + \frac{r^{2}
\norm{u}^{2}}{n}
\end{equation*}
Hence, the following will be satisfied:
\begin{equation}
\label{eq:perceptron-bound} \expectation{\mathcal{L}_{0/1}\left(w\right)} \le
\mathcal{L}\left(u\right)+\epsilon
\end{equation}
when:
\begin{equation*}
n\le
O\left(\left(\frac{\mathcal{L}\left(u\right)+\epsilon}{\epsilon}\right)\frac{r^{2}
\norm{u}^{2}}{\epsilon}\right)
\end{equation*}
The expectation is taken over the random sampling of $w$. The number of kernel
evaluations performed by the $i$th iteration of the Perceptron will be equal to
the number of mistakes made before iteration $i$.  This quantity is upper
bounded by the total number of mistakes made over $n$ iterations, which is
given by the mistake bound of equation \ref{eq:perceptron-starting-bound}:
\begin{align*}
\abs{\mathcal{M}}\le & n\mathcal{L}\left(u\right) +
r\norm{u}\sqrt{n\mathcal{L}\left(u\right)} + r^{2}\norm{u}^{2}\\
\le &
O\left(\left(\frac{1}{\epsilon}\left(\frac{\mathcal{L}\left(u\right)+\epsilon}{\epsilon}\right)\mathcal{L}\left(u\right)\right.\right.\\
& +
\left.\left.\sqrt{\frac{1}{\epsilon}\left(\frac{\mathcal{L}\left(u\right)+\epsilon}{\epsilon}\right)\mathcal{L}\left(u\right)}+1\right)r^{2}\norm{u}^{2}\right)\\
\le &
O\left(\left(\left(\frac{\mathcal{L}\left(u\right)+\epsilon}{\epsilon}\right)^{2}-\left(\frac{\mathcal{L}\left(u\right)+\epsilon}{\epsilon}\right)\right.\right.\\
& \left. +
\sqrt{\left(\frac{\mathcal{L}\left(u\right)+\epsilon}{\epsilon}\right)^{2}-\left(\frac{\mathcal{L}\left(u\right)+\epsilon}{\epsilon}\right)}+1\right)\\
& \left. \cdot r^{2}\norm{u}^{2}\right)\\
\le &
O\left(\left(\frac{\mathcal{L}\left(u\right)+\epsilon}{\epsilon}\right)^{2}r^{2}\norm{u}^{2}\right)
\end{align*}
The number of mistakes $\abs{\mathcal{M}}$ is necessarily equal to the size of
the support set of the resulting classifier. Substituting this bound into the
number of iterations:
\begin{align*}
\mbox{\#K}= & n\abs{\mathcal{M}}\\
\le &
O\left(\left(\frac{\mathcal{L}\left(u\right)+\epsilon}{\epsilon}\right)^{3}\frac{r^{4}\norm{u}^{4}}{\epsilon}\right)
\end{align*}
This holds in expectation, but we can turn this into a high-probability result
using Markov's inequality, resulting in in a $\delta$-dependence of
$\frac{1}{\delta}$.
\end{proof}

Although this result has a $\delta$-dependence of $1/\delta$, this is merely a
relic of the simple online-to-batch conversion which we use in the analysis.
Using a more complex algorithm (e.g. \citet{CesaCoGe01}) would
likely improve this term to $\log\frac{1}{\delta}$.

\section{Convergence rates of dual optimization methods}\label{app:dual-decomposition}

In this section we discuss existing analyses of dual optimization
methods. We first underscore possible gaps between dual sub-optimality
and primal sub-optimality. Therefore, to relate existing analyzes in
the literature of the dual sub-optimality, we must find a way to
connect between the dual sub-optimality and primal sub-optimality. We
do so using a result due to \citet{ScovelHuSt08}, and based on this
result, we derive convergence rates on the primal sub-optimality.

Throughout this section, the ``SVM problem'' is taken to be the regularized
objective of Problem \ref{eq:regularized-objective}.  We denote the primal
objective by:
\begin{equation*}
P(w) = \frac{\lambda}{2} \|w\|^2 + \frac{1}{n} \sum_{i=1}^n
\ell(y_i\inner{w,x_i})
\end{equation*}
The dual objective can be written as:
\begin{equation*}
D(\alpha) = {\lambda} \left( \sum_{i=1}^n \alpha_i - \frac{1}{2} \sum_{i,j=1}^n \alpha_i
\alpha_j Q_{ij} \right)
\end{equation*}
where $Q_{ij} = y_i y_j \inner{x_i}{x_j}$, and the dual constraints
are $\alpha \in [0,1/(\lambda n)]^n$. Finally, by strong duality we have:
\begin{equation*}
P^* = \arg\max_w P(w) = 
\argmax{\alpha \in [0,1/(\lambda n)]^n} D(\alpha) = D^*
\end{equation*}

\subsection{Dual gap vs. Primal gap}

Several authors analyzed the convergence rate of dual optimization algorithms.
For example, \citet{HsiehChLiKeSu08,CollinsGlKoCaBa08} analyzed the convergence
rate of SDCA and \citet{chen2006study} analyzed the convergence rate of
SMO-type dual decomposition methods. In both cases, the number of iterations
required so that the dual sub-optimality will be at most $\epsilon$ is
analyzed. This is not satisfactory since our goal is to understand how many
iterations are required to achieve a \emph{primal} sub-optimality of at most
$\epsilon$.  Indeed, the following lemma shows that  a guarantee on a small
dual sub-optimality might yield a trivial guarantee on the primal
sub-optimality.
\medskip
\begin{lem}
\label{lem:psodso}
For every $\epsilon > 0$, there exists a SVM problem with a dual solution $\alpha$
that is $\epsilon$-accurate, while the corresponding primal solution, $w =
\sum_i \alpha_i y_i x_i$, is at least $(1-\epsilon)$ sub-optimal. Furthermore,
the distribution is such that there exists $u$ with $\|u\|=1$ and
$\mathcal{L}(u) = 0$, while $\mathcal{L}(w)=1$ and $\mathcal{L}_{0,1}(w) =
1/2$. 
\end{lem}
\begin{proof}
Fix some $u$ with $\|u\|=1$ and choose any distribution such that
$\mathcal{L}(u) = 0$. Take a sample of size $n$ from this distribution. A
reasonable choice for the regularization parameter of SVM in this case is to
set $\lambda = 2\epsilon$. We have: $ P^* \leq P(u) = \tfrac{\lambda}{2}
\|u\|^2 = \epsilon$.  Now, for $\alpha = 0$ we have $ D^* - D(\alpha) = P^* - 0
\leq \epsilon$. Therefore, the dual sub-optimality of $\alpha=0$ is at most
$\epsilon$. On the other hand, the corresponding primal solution is $w = 0$,
which gives $ P(0) - P^* = 1 - P^* \geq 1 - \epsilon$. Furthermore,
$\mathcal{L}(0)=1$ and $\mathcal{L}_{0,1}(0) = 1/2$, assuming that we break
ties at random. 
\end{proof}

In an attempt to connect between dual and primal sub-optimality,
\citet{ScovelHuSt08} derived approximate duality theorems. This was used
by \citet[Theorem 2]{hush2006qp} to show the following:
\medskip
\begin{thm}
\label{thm:hush}
\emph{\citep[Theorem 2]{hush2006qp}}
To achieve $\epsilon_p$ sub-optimality in the primal, it suffices to require a
sub-optimality in the dual of $\epsilon \le \frac{\lambda\,\epsilon_p^2}{118}$.
\end{thm}
There is no contradiction to Lemma \ref{lem:psodso} above since in the proof of
the lemma we set $\lambda = 2 \epsilon$, which yields $\epsilon_p \ge 1$.
 
\subsection{Analyzing the primal sub-optimality of dual methods}

\citet{chen2006study} derived the linear convergence of SMO-type algorithms.
However, the analysis takes the following form:
\begin{quote}
There are $c<1$ and $\bar{k}$, such that for all $k \ge \bar{k}$ it holds that
$D(\alpha^{(k+1)})-D^* \le c(D(\alpha^{(k)})-D^*)$. 
\end{quote}
In the above, $\alpha^{(k)}$ is the dual solution after performing $k$
iterations, and $D^*$ is the optimal dual solution. 

This type of analysis is not satisfactory since $\bar{k}$ can be extremely
large and $c$ can be extremely close to $1$. As an extreme example, suppose
that $\bar{k}$ is exponential in $n$. Then, in any practical implementation of
the method, we will never reach the regime in which the linear convergence
result holds. As a less extreme example, suppose that $\bar{k} \ge n^2$. It
follows that we might need to calculate the entire Gram matrix before the
linear convergence analysis kicks in. To make more satisfactory statements, we
therefore seek convergence analyses which demonstrate good performance not only
asymptotically, but also for reasonably small values of $k$.

\citet{hush2006qp} combined explicit convergence rate analysis of the dual
sub-optimality of certain decomposition methods with Theorem \ref{thm:hush}.
The end result is an algorithm with a bound of $O(n)$ on the number of dual
iterations, and a total number of kernel evaluations at training time of
$O(n^2)$. It also follows that the number of support vectors can be order of
$n$. 

\citet{HsiehChLiKeSu08} analyzed the convergence rate of SDCA and derived a
bound on the duality sub-optimality after performing $T$ iterations.
Translating their results to our notation and ignoring low order terms we
obtain:
\begin{equation*}
\epsilon_D \le \frac{n}{T+n} \left( (\lambda/2) \|\alpha^*\|^2 + P^*\right) ~.
\end{equation*}
where $\alpha^*$ is such that $w^* = \sum_i \alpha^*_i y_i x_i$.
Combining this with Theorem \ref{thm:hush} yields that the number of
iterations, according to this analysis, should be at least
\begin{equation*}
T \ge \Omega\left(\frac{n P^*}{\lambda \epsilon_P^2} \right)~.
\end{equation*}
So, even if we set $\epsilon_P = P^*$ we still need
\begin{equation*}
T \ge \Omega\left(\frac{n}{\lambda \epsilon_P} \right)~.
\end{equation*}
Each iteration of SDCA cost roughly the same as a single iteration of Pegasos.
However, Pegasos needs order of $1/(\lambda \epsilon_P)$ iterations, while
according to the analysis above, SDCA requires factor of $n$ more iterations.
We suspect that this analysis is not tight.

\end{document}